\documentclass[journal]{IEEEtran}
\ifCLASSINFOpdf
\else
\fi
\usepackage{bm}
\usepackage{color}
\usepackage{amsmath,multicol, amsthm}
\usepackage{amssymb}
\usepackage{graphicx}
\usepackage[boxruled, linesnumbered]{algorithm2e}
\usepackage{amsthm,balance} 
\usepackage[colorinlistoftodos]{todonotes}

\newtheorem*{remark}{Remark}
\newtheorem{theorem}{Theorem}

\begin{document}
\title{Uncertainty Estimation of Dense Optical-Flow for Robust Visual Navigation} 

\author{Yonhon Ng, Hongdong Li and Jonghyuk Kim % stops a space
\thanks{Y. Ng and H. Li are with the Research School of Engineering, Australian National University, Acton ACT 2601, Australia e-mail: \{yonhon.ng, jonghyuk.kim, hongdong.li\}@anu.edu.au. J. Kim is with University of Technology Sydney, e-mail: jonghyuk.kim@uts.edu.au.}}

\maketitle
%\IEEEpeerreviewmaketitle

\begin{abstract}
This paper presents a novel dense optical-flow algorithm to solve the monocular simultaneous localization and mapping (SLAM) problem for ground or aerial robots. Dense optical flow can effectively provide the ego-motion of the vehicle while enabling collision avoidance with the potential obstacles. Existing work has not fully utilized the uncertainty of the optical flow -- at most an isotropic Gaussian density model. We estimate the full uncertainty of the optical flow and propose a new eight-point algorithm based on the statistical Mahalanobis distance. Combined with the pose-graph optimization, the proposed method demonstrates enhanced robustness and accuracy for the public autonomous car dataset (KITTI) and aerial monocular dataset. 
\end{abstract}


% \fi
\IEEEpeerreviewmaketitle

\section{Introduction}

Unmanned aerial vehicles (UAVs) have drawn significant attention from the research community and industry in the last few decades. The primary advantage of a UAV is the ability to access places that are hazardous and hard to reach, such as for inspection of infrastructure~\cite{Morgenthal14}\cite{Nikolic13}, precision agriculture~\cite{Herwitz02}\cite{Gago15}, and search and rescue operation~\cite{Bernard11}\cite{Waharte10}. Due to the lightweight, rich information, and low cost, cameras have been extensively applied for robot navigation and environment mapping. In particular, visual simultaneous localization and mapping (SLAM) algorithms have drawn significant interest in both robotics and computer vision communities~\cite{Kummerle11}\cite{Stachniss17}. These algorithms typically consist of a front-end and a back-end part. The front-end part typically performs visual odometry between consecutive images or relative to a known map, and the back-end module performs graph optimization as well as handling the loop-closures. 

\begin{figure}[t]\centering
	{\includegraphics[width=0.9\columnwidth]{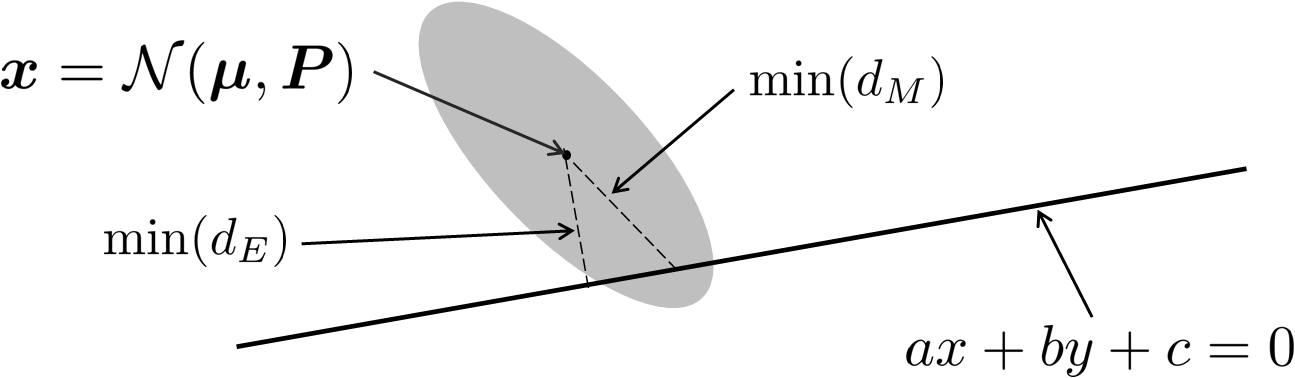}}
	\caption{\label{fig:Mahalanobis_point2line} An illustrative figure showing an image feature pixel $\bm{x}$ represented as a 2-dimensional random variable with mean $\mu$ and covariance matrix $\bm{P}$, the epipolar line is represented as a straight line $l$ with equation $a x + b y + c = 0$, $\min(d_M)$ is the minimum Mahalanobis distance, while $\min(d_E)$ is the minimum Euclidean distance. In standard $8$-point algorithms, the (inverse) Euclidean distance is used as the weight, while our method utilizes the Mahalanobis distance which is statistically more consistent (shorter distance to the line in this example).}
\end{figure}

% VO
Most monocular visual odometry methods use sparse feature points matched between images to compute the inter-frame motion \cite{Klein07}\cite{Song16}\cite{Fanani17}. The feature matching accuracy is improved by incorporating the kinematic vehicle models \cite{Bradler15}\cite{Fanani17}. \cite{Song16} adopts the learning-based method, and \cite{Fanani17} utilizes the convolutional neural network to train the ground plane detection to estimate the height for a ground vehicle. These approaches, however, make the methods not suitable for aerial or rough terrain applications. 

Recently, dense optical flow for monocular visual odometry has received significant attention, and the current state-of-the-art methods are capable of producing on average, more than $85$\% of dense correspondences having less than $3$ pixels error~\cite{Xu17}\cite{Hui18}\cite{Sun18}. Although the computed correspondences from optical flow are not very accurate compared to sparse feature matches, the dense nature of the correspondences helps mitigate the inaccuracy of individual matches. Another benefit of using dense optical flow is that it can avoid the shortcomings of typical sparse matches. For example, the sparse matches may be clustered around a small area of the image or encounter problems with the planar degeneracy~\cite{Hartley03}, resulting in a biased motion estimate. However, these methods have not explicitly considered the uncertainty in the feature matching and monocular SLAM pipelines. 

In this paper, we propose a novel \textit{robust} monocular simultaneous localization and mapping in a principled probabilistic framework. This is accomplished by using dense optical flow with estimated uncertainty as to the input to our visual odometry pipeline. Utilizing the existing robust back-end pose-graph SLAM \cite{Cheng15}, the methods will achieve significant robustness to the sensing uncertainty and loop-closure outliers. The contributions of this paper are threefold:
\begin{itemize}
    \item Uncertainty is estimated from the dense optical flow. The epipolar constraint is first included in the matching cost to improve the matching accuracy. The uncertainty is then recovered by fitting a bivariate Gaussian to the matching cost.
    \item A new Mahalanobis eight-point algorithm is developed to compute the visual odometry. The estimated uncertainty enables efficient sampling of RANSAC (random sampling and consensus) and accurate pose estimation using the weights based on the Mahalanobis distance. 
    \item Demonstration of the proposed methods for ground and aerial navigation. 
\end{itemize}

To the best of our knowledge, this is the first work that achieves the robustness for both the front- and back-end SLAM and its application for aerial navigation in an unstructured forest environment.  

The paper is organized as follows. Section \ref{sec:relatedwork} discusses related works on optical flow and visual odometry. Section~\ref{sec:BDFlow} presents our bivariate uncertainty estimation method from the measured optical flow. Section \ref{sec:Weighted8Pt} details our new eight-point algorithm by introducing the Mahalanobis distance from the epipolar line and using it as the weighting factor in the eight-point algorithm. Section \ref{sec:pipeline} describes the visual processing pipeline of the visual SLAM, discussing the scale estimation, inter-frame pose estimation, depth estimation, and loop closure computation. Section \ref{sec:Experiment} shows the experimental results that verify the performance of our proposed method, and Section \ref{sec:Conclusion} presents the conclusions and future extensions of our work. 

%%%%%%%%%%%%%%%%%%%%%%%%%%%%
\section{Related Work} \label{sec:relatedwork}

Due to the lack of a fixed baseline, visual odometry using a monocular (single camera) system has the inherent difficulty in estimating the metric distance, unlike stereo camera~(\emph{e.g.} \cite{Mur-Artal16}\cite{Engel15}) setup. For the same reason, monocular (single camera) visual odometry also has difficulty with scale drift. Another challenge of visual odometry is in the image feature matching step, especially for scenes with repetitive texture and dynamic objects. 
Work by \cite{Cheviron07} and \cite{Artieda09} also assume the scene contains known objects with known dimensions to assist with the visual odometry task. However, this assumption is not valid when exploring an unknown, unstructured environment.

State-of-the-art monocular visual odometry methods suitable for a forward-looking camera in an unstructured environment are primarily designed for ground vehicle \cite{Geiger11}\cite{Song16}\cite{Fanani17}. These methods prevent scale drift and recover the metric scale by assuming that the camera travels at a fixed, known height above a roughly planar ground. Although this is a valid assumption for a ground-based vehicle, it is not true for an aerial vehicle. 

Dense optical flow methods have been investigated to address the difficulty of feature matching in low or repetitive textured scenes with dynamic objects. Some notable work includes the use of full discrete cost volume to compute the dense optical flow~\cite{Chen16}\cite{Xu17} directly. DCFlow~\cite{Xu17} is one such work in which an efficient semi-global matching is proposed. DCFlow also utilizes EpicFlow~\cite{Revaud15} for interpolation and local homography fitting to obtain accurate optical flow estimation with lower computational cost. Recently, convolutional neural network, pyramidal feature extraction and feature warping have also been applied successfully to compute the optical flow~\cite{Hui18}\cite{Sun18}. However, these methods, like most state-of-the-art optical flow methods, do not calculate and utilize the uncertainty of the optical flow. In comparison, our proposed method directly make use of the full discrete cost volume of DCFlow~\cite{Xu17}, which allows simple incorporation of extra matching cost (epipolar) constraint to improve the optical flow accuracy and direct estimation of the 2D uncertainty. 

The dense optical flow correspondences are treated like sparse feature matches, which are used to estimate the so-called Fundamental matrix. The well-known normalized eight-point algorithm~\cite{Hartley97} can be used to compute the Fundamental matrix. The conventional eight-point algorithm does not minimize a meaningful geometrical error. Thus, other methods that minimize the Sampson distance~\cite{Torr98}, symmetric epipolar distance~\cite{Zhang98} and re-projection error~\cite{Hartley03} were developed. A review paper on other methods to compute the Fundamental matrix is given in~\cite{Armangue03}. However, these methods usually assume isotropic, homogeneous error in the image features location, which is not suitable for optical flow correspondences. Due to the definition of optical flow, the feature location error is zero in the first image. In general, it has an anisotropic and non-homogeneous error in the feature location in the second image. Based on this knowledge, we derive a new method to compute the Fundamental matrix that minimizes the Mahalanobis distance of the image feature in the second image frame. Using a calibrated camera, the inter-frame camera pose can then be recovered from the Fundamental matrix.

%%%%%%%%%%%%%%%%%%%%%%%%%%%%%%%%%%%%%%%%%%%%%%%%%
\section{Dense Flow with Uncertainty}\label{sec:BDFlow}

A new optical flow method is developed based on DCFlow~\cite{Xu17}. The accuracy of the optical flow estimate is improved by incorporating the epipolar constraint into the cost computation, and the optical flow uncertainty is extracted by fitting a Gaussian function to the cost volume. 

\begin{figure}[t]
\centering
	{\includegraphics[width=\columnwidth]{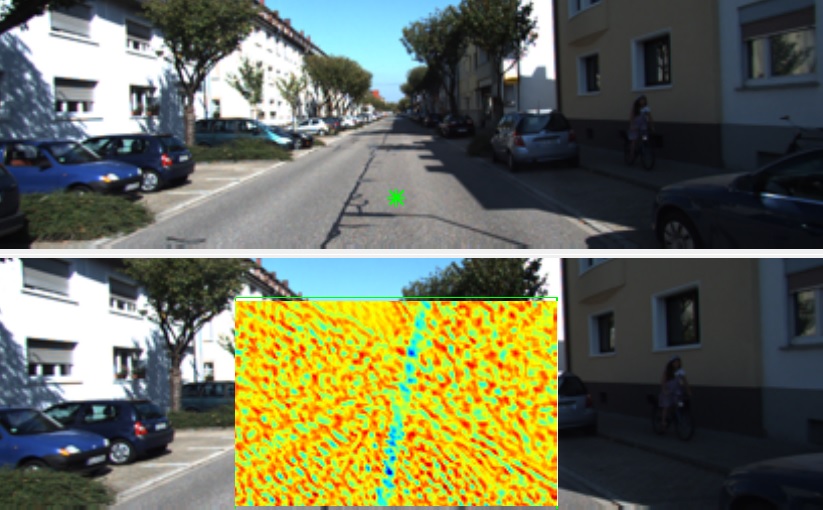}}
\caption{\label{fig:costSlice} An example illustrating the epipolar constraint added to a cost slice from two consecutive images. (top) The first image shows a pixel marked by a green star. (bottom) The second image shows a bounding box enclosing the candidate matching pixels for the corresponding pixel in the first image, with the matching cost of corresponding candidate pixels. Red color represents low matching candidates and blue for high matching. The epipolar line (a straight line towards an epipole) shows highly likely matching.}
\end{figure}

\subsection{Dense Flow with Epipolar Constraint}

Dense optical flow is a method to estimate the motion of each pixel between two input images. The algorithm typically optimizes a cost function consisting of a data cost that penalizes visually dissimilar pixels and a regularisation cost that encourages spatially smooth variation of the optical flow field. For each pixel in the first image of $M \times N$, the matching cost of a set of candidate pixels of $D \times D$ in the second image is computed, such that the full cost volume is four-dimensional (4D).

The epipolar constraint is added to the cost function before regularisation is applied. Figure~\ref{fig:costSlice}(bottom) shows an example of the epipolar constraint added to one of the cost volumes slice, which encourages the correspondences to be close to the epipolar line by increasing the cost of finding a match far from the line. This process is accomplished as follows. First, Shi-Tomasi corner features tracked by Kanade-Lucas-Tomasi (KLT) algorithm \cite{Shi94}, are used as sparse correspondences for the well known eight-point algorithm~\cite{Hartley97} to obtain an initial estimate of the Fundamental matrix. A truncated $L_2$ cost is added to the cost volume to enforce the epipolar constraint based on the computed Fundamental matrix. When the pixel in the first image corresponds to a static point of the scene, the cost of finding the match far away from the epipolar line is increased proportionately to squared distance. Conversely, when a pixel in the first image corresponds to a point on a moving object, a truncated cost is applied. This truncation helps to avoid matches that satisfy the epipolar constraint but are visually dissimilar to be wrongly selected.

%%%%%%%%%%%%%%%%%%%%%%%%%%%%%%%%%%%%%%%%%%%%%%%%
\subsection{Uncertainty Estimation}\label{sec:OF_uncertainty}

DCFlow~\cite{Xu17} and most other state-of-the-art optical flow methods implicitly assume each correspondence has a homogeneous, isotropic Gaussian uncertainty. However, the uncertainty of each correspondence can have different magnitude and correlation depending on the visual similarity of neighboring pixels. Figure~\ref{fig:uncertainty} illustrates an example of a matching cost slice, in which the negative logarithm of a unimodal Gaussian distribution is fitted to the optical flow matching cost output.

\begin{figure}[t]
\centering
	{\includegraphics[width=1\columnwidth]{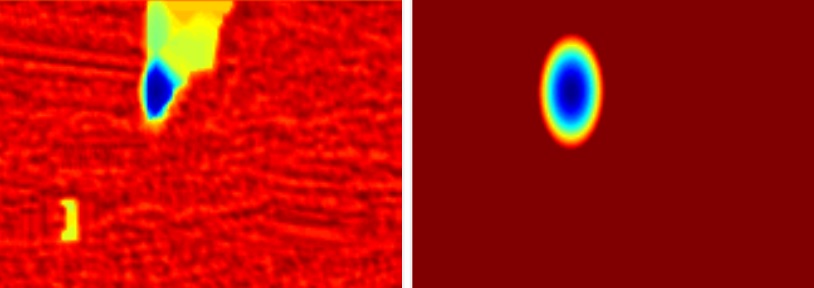}}
	\caption{\label{fig:uncertainty} Uncertainty fitting of negative logarithm of a bivariate Gaussian to a matching cost slice (after spatial smoothness regularisation step). From left to right: (a) 2D cost slice, (b) the approximate $2$D cost slice using a $2$D Gaussian fitting. }
\end{figure}

For a general two-dimensional Gaussian distribution, we know that the negative logarithm of the likelihood function is half of the squared Mahalanobis distance. The squared Mahalanobis distance, $d_M^2$ can be computed as~\cite{Mahalanobis36}, 
\begin{align}
\label{eq:Mahalanobis_dist}
{d_M}^2(\bm{x}| \bm{\mu}, \bm{Y}) &= (\bm{x} - \bm{\mu})^T \bm{Y} (\bm{x} - \bm{\mu}).
\end{align}
We represent the vector $(\bm{x} - \bm{\mu}) = [x, y]^T$, and
\begin{align}
{d_M}^2 &= \begin{bmatrix}
{x} & {y}
\end{bmatrix}
\begin{bmatrix}
{Y}_{xx}   & {Y}_{xy}  \\ 
{Y}_{xy} &  {Y}_{yy} 
\end{bmatrix}
\begin{bmatrix}
{x} \\
{y}
\end{bmatrix} \\
&= {Y}_{xx} {{x}}^2 + 2 \tilde{Y}_{xy} {x} {y} + {Y}_{yy} {{y}}^2,
\end{align}
where $\bm{x}$ is the vector representing the coordinates of a point, $\bm{\mu}$ is the vector representing the coordinates of the mean (optical flow estimate) of the Gaussian distribution, and $\bm{Y}$ is the information matrix of the Gaussian distribution which is the inverse of the covariance $\bm{Y}= \bm{P}^{-1}$. 

The elements of information matrix, $\bm{Y}$ can then be computed using the least squares as
\begin{equation}
\label{eq:iCov_est}
\begin{split}
\underbrace{\begin{bmatrix}
{{x}_1}^2 & 2{x}_1 {y}_1 & {{y}_1}^2 \\
{{x}_2}^2 & 2{x}_2 {y}_2 & {{y}_2}^2 \\
 & \vdots & \\
{{x}_N}^2 & 2{x}_N {y}_N & {{y}_N}^2 
\end{bmatrix} }_{\bm{A}}
\underbrace{\begin{bmatrix}
{Y}_{xx} \\
{Y}_{xy} \\
{Y}_{yy}
\end{bmatrix}}_{\bm{y}} &= \underbrace{\begin{bmatrix}
{d_1}^2 \\
{d_2}^2 \\
\vdots \\
{d_N}^2
\end{bmatrix}}_{\bm{d}} \\
%\bm{A} \;  \bm{Y} \hspace{0.2cm} &= \hspace{0.2cm} \bm{d} \\
\therefore\bm{y} &= (\bm{A}^T \bm{A})^{-1} \bm{A}^T \bm{d},
\end{split}
\end{equation}
%{\color{red} Q) Hon: How the $d^2_M$ is measured?}
where $\{d_1, d_2, \cdots, d_N\}$ are the matching cost at their respective $x$ and $y$ coordinates from the optical flow cost volume. DCFlow computes the matching cost efficiently by using $(1-f_1 \cdot f_2)$, where $f_1$ and $f_2$ are the unit vectors representing the image feature descriptor. This method results in a cost value between $0$ (visually similar) and $1$ (visually dissimilar). However, the negative logarithm of a Gaussian likelihood function has a value between $0$ to infinity. Thus, we can exclude pixels with high cost from our Gaussian fitting by only using pixels with a matching cost below a set threshold. 

Similar to most top performance optical flow methods, DCFlow has a post-processing step to remove unreliable matches from the semi-dense correspondences before EpicFlow \cite{Revaud15} interpolation. This step is accomplished by computing the forward and backward optical flow and eliminating those matches that do not satisfy the forward-backward consistency. This post-processing step changes the uncertainty estimate such that the correspondences that got removed should be assigned a high uncertainty. We replace those values with the maximum uncertainty of the optical flow estimate. 

These provide us with three channels (${Y}_{xx}, {Y}_{xy}, {Y}_{yy}$), encoding the information matrix for every pixel correspondences for the scaled-down pair of RGB input images. We can scale the uncertainty image back to the original resolution by applying an image resize operation. First, the information matrix parameters are converted to covariance parameters, scaled up to the original image resolution, followed by a multiplication of 9 (squared of image rescaling factor). The scaled-up covariance parameters are then converted back to the information matrix following matrix inverse. 

% Covariance matrices are symmetrical, positive semi-definite matrices. We can prove that positive scalar addition of two covariance matrices is also a covariance matrix as follows. 
% \begin{proof}
% Let $\bm{A}$ and $\bm{B}$ be $n \times n$ positive semi-definite matrices, $a$ and $b$ be some positive scalar values, 

% \begin{equation*}
% v^T \bm{A} v \geq 0, \hspace{1cm} v^T \bm{B} v \geq 0, \hspace{1cm} \forall v \in \mathbb{R}^n
% \end{equation*}
% \begin{equation*}
% v^T (a\bm{A}) v \geq 0, \hspace{1cm} v^T (b\bm{B}) v \geq 0
% \end{equation*}
% \begin{equation*}
% v^T (a\bm{A}) v + v^T (b\bm{B}) v \geq 0
% \end{equation*}
% \begin{equation*}
% v^T (a\bm{A} + b\bm{B}) v \geq 0
% \end{equation*}
% \end{proof}

The estimated uncertainty can also be used to determine if the two input images are visually similar, as illustrated in Fig.~\ref{fig:OF_uncertainty}, which will be helpful in computing the loop closure constraints (discussed in Section~\ref{sec:pipeline}). If the two input images belong to the same scene, most local neighbors will have similar optical flow magnitude and direction. The regularisation step will then shrink the region of possible matching locations, and thus, the uncertainty decreases. On the other hand, local neighbors generally have different optical flow magnitude and directions if the two input images belong to different scenes. In that case, the regularisation step will not shrink the region of possible matching locations, and the uncertainty will become high. 

\begin{figure}[t]
\centering
	{\includegraphics[width=0.49\columnwidth]{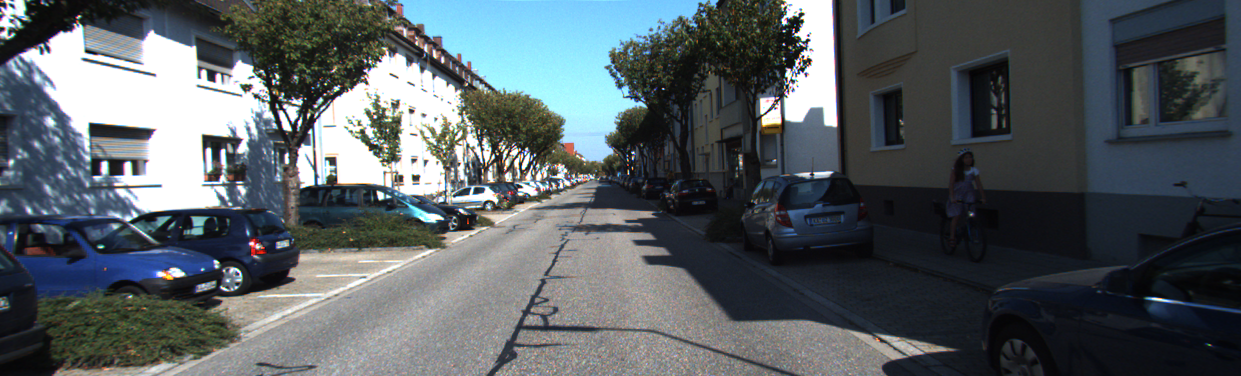}}
	{\includegraphics[width=0.49\columnwidth]{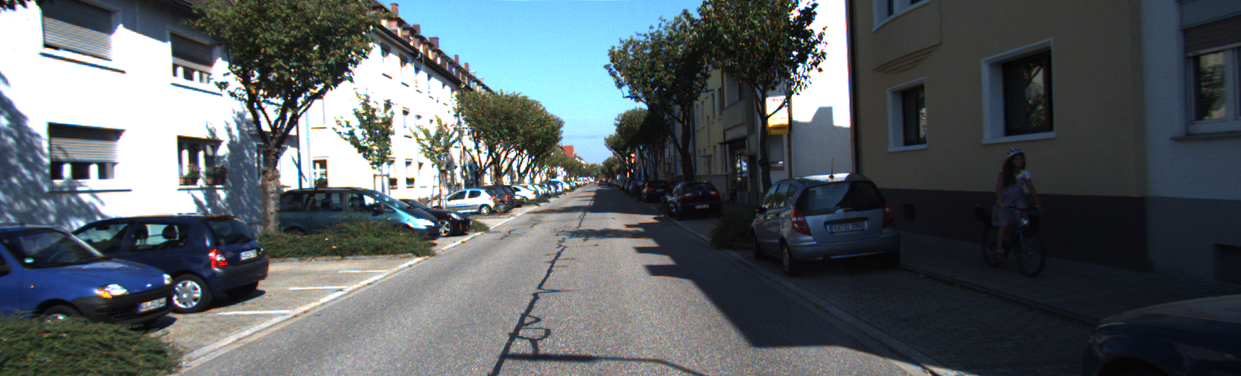}}
	{\includegraphics[width=0.49\columnwidth]{figs/004540.png}}
	{\includegraphics[width=0.49\columnwidth]{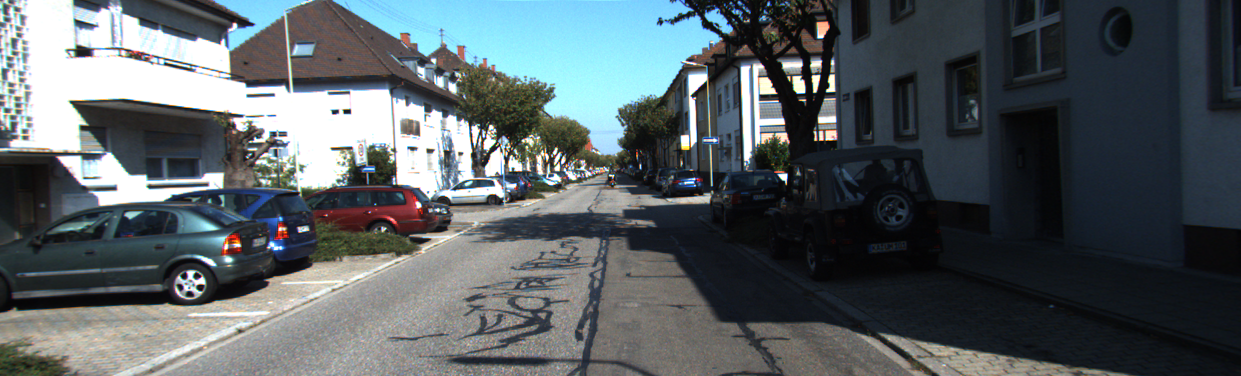}}
	{\includegraphics[width=0.49\columnwidth]{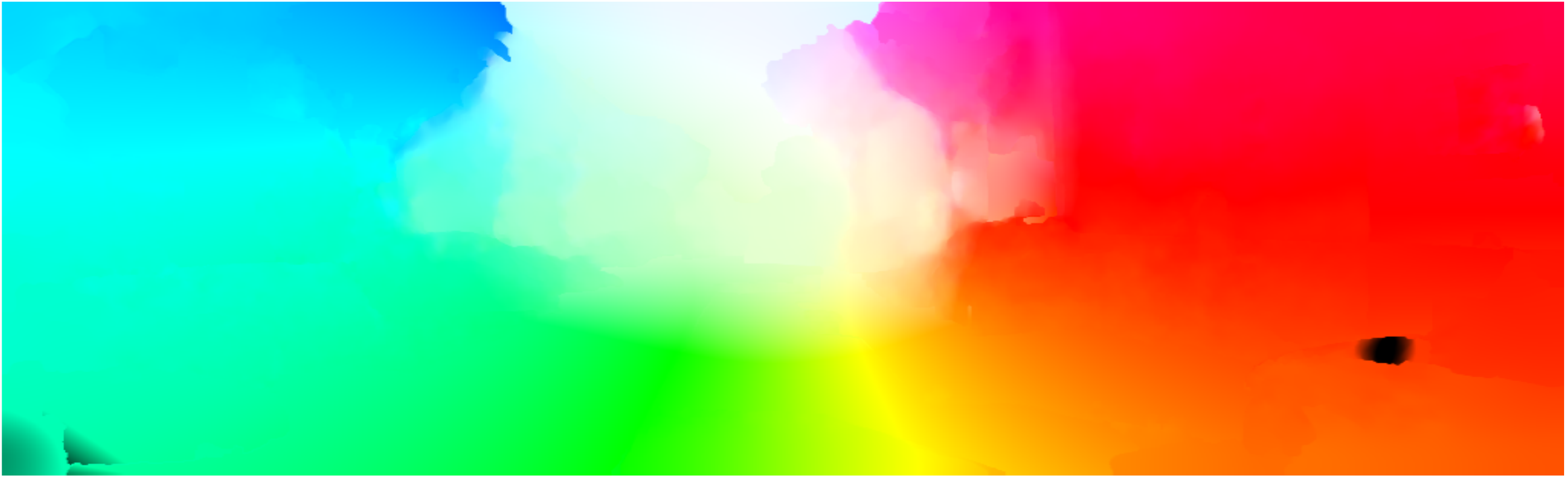}}
	{\includegraphics[width=0.49\columnwidth]{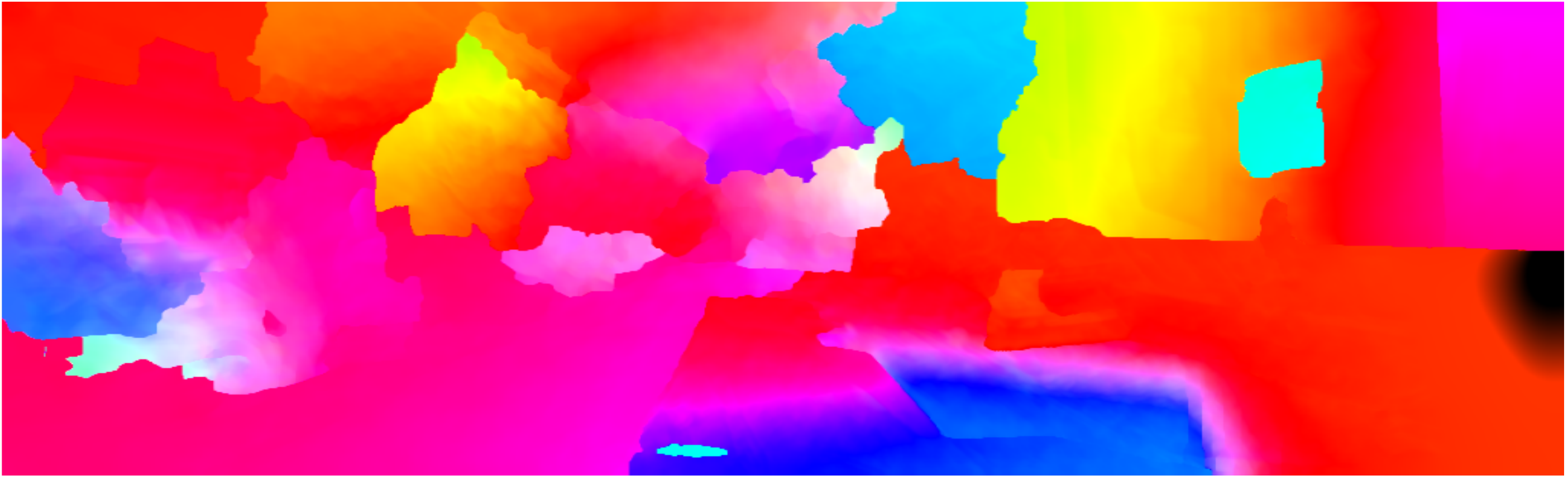}}
	{\includegraphics[width=0.49\columnwidth]{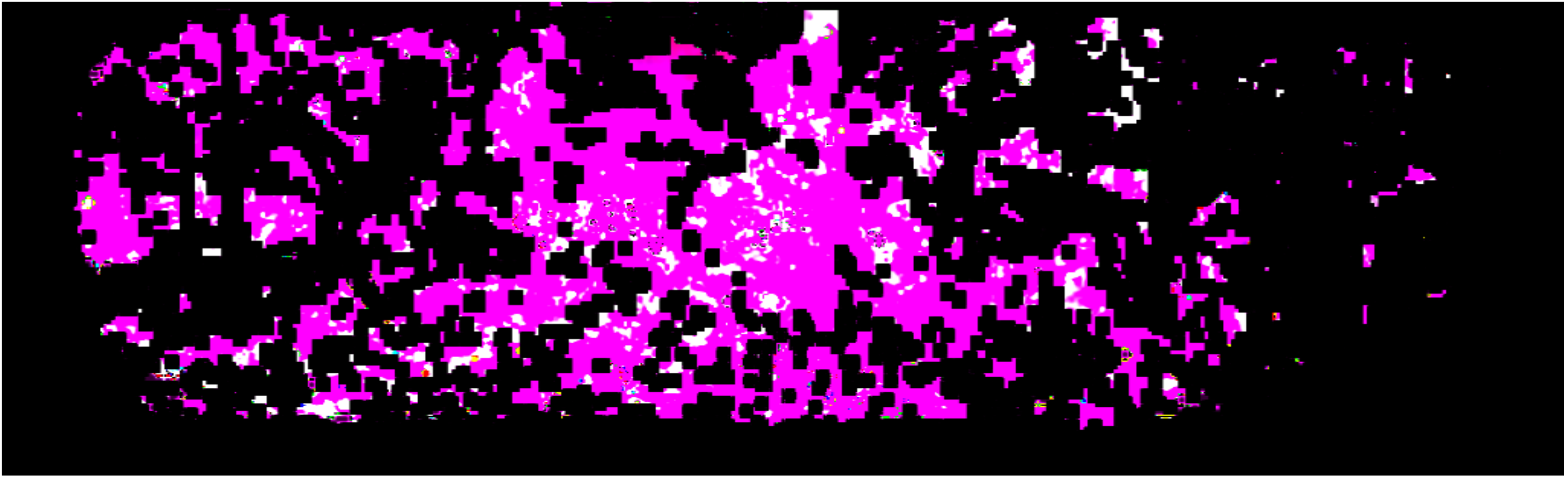}}
	{\includegraphics[width=0.49\columnwidth]{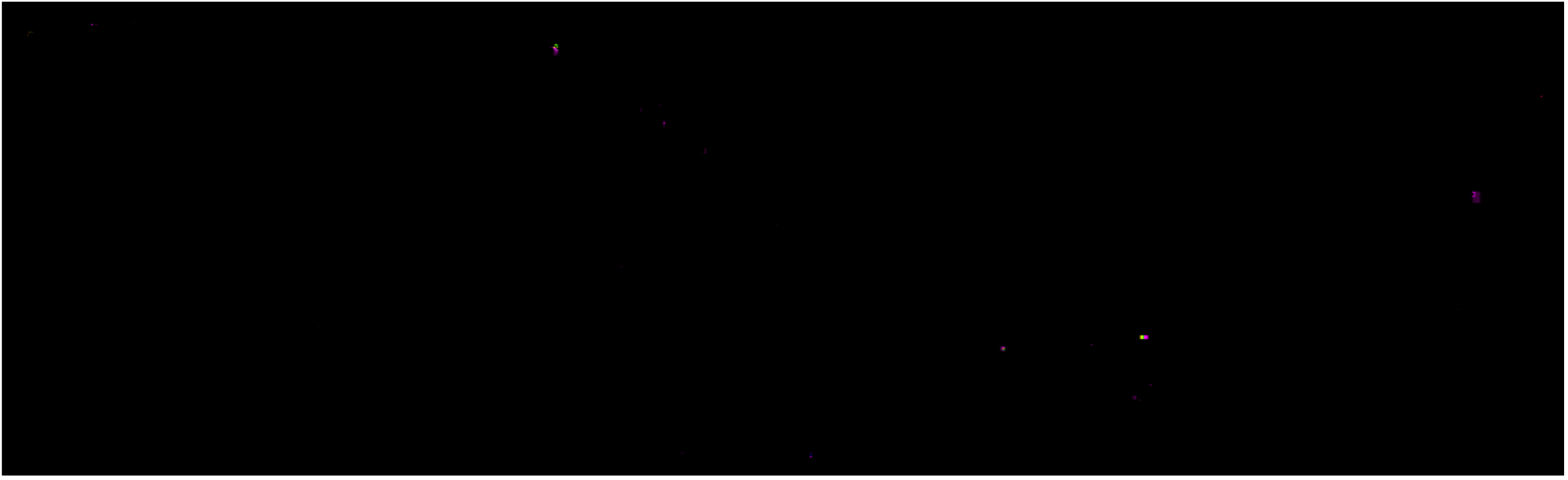}}
	\caption{\label{fig:OF_uncertainty} An example of estimated optical flow and uncertainty magnitude. From top to bottom: first input image, second input image, optical flow, estimated information matrix. Left column corresponds to sequential images, while right column corresponds to two input images with high structural similarity (SSIM) index, but is not of the same scene. Black colour for information matrix values corresponds to high covariance (unreliable) pixels. }
\end{figure}

%%%%%%%%%%%%%%%%%%%%%%%%%%%%%%%%%%%%%%%%%%%%%%%
%A new 3D scene point triangulation method based on MIGE (Chapter~\ref{cha:MIGE}) is also presented, which estimates both the mean and information matrix. 
% The method to accurately recover the scale of the motion is also presented. The accurate inter-frame motion estimate is obtained from fusing the Mahalanobis eight-point algorithm result and perspective-n-points result. 

%A new method to efficiently fuse two given depth map (triangulated scene points) is also presented, where the previous depth map estimate is fused with the current estimate to obtain more accurate depth map result. The depth map is then propagated to the next frame for future computation. 

%Ways to determine and handle small motion in the video sequences are also discussed in the following subsections. 

%\begin{remark}
%Note that unless otherwise stated, the estimated poses, reconstructed 3D points and 3D points uncertainty are all expressed with respect to the previous frame. For example, at current time, a new image frame with index $t$ is captured, we fix the coordinate frame at the pose of frame $t-1$, with $z$-axis pointing forward, $x$-axis points to the right, $y$-axis points downward, and the origin at the centre of the camera at frame $t-1$. 
% \end{remark}

\section{New Mahalanobis 8-points Algorithm}\label{sec:Weighted8Pt}
%\section{Motion Estimation}
%\subsubsection{Mahalanobis eight-point algorithm}\label{sec:pose_est}
The computed dense optical flow, dense depth estimate (prior) and their uncertainty are used to estimate the camera motion, called visual odometry. The dense correspondences are used similar to conventional sparse feature matches. In contrast, the optical flow uncertainty is used for the random sampling and consensus and the Mahalanobis eight-point algorithm. 

From a pair of input images, we can find a set of matching pixels $\bm{x}_i \leftrightarrow {\bm{x}_i}'$. Then, the two pixels and the $3$D feature location stay in a plane called the epipolar plane. Equivalently, a fundamental matrix $\bm{F}$ exists, 
\begin{equation}
{\bm{x}_i'}^T \bm{F} \bm{x}_i = 0,
\end{equation} 
where $\bm{x}_i$ and ${\bm{x}_i}'$ are represented in the homogeneous coordinates. 

Each matching pixel provides a linear constraint on the elements of $\bm{F}$. Since the scale of $\bm{F}$ can be arbitrary, the solution of $\bm{F}$ can be computed using $8$ sets of matching pixels. The Fundamental matrix can be represented by a vector ($\bm{f}$) of length $9$ to solve the equations. Given $n$ pairs of matching image features, the linear constraints can be concatenated into a matrix form as
\begin{align}
\label{eq:C4_fund_linear2}\centering
%\bm{A} \bm{f} = 0 \nonumber\\
\underbrace{\begin{bmatrix}
x'_1 x_1 & x'_1 y_1 & x'_1 & y'_1 x_1 & y'_1 y_1 & y'_1 & x_1 & y_1 & 1 \\
\vdots & \vdots & \vdots & \vdots & \vdots & \vdots & \vdots & \vdots & \vdots \\
x'_n x_n & x'_n y_n & x'_n & y'_n x_n & y'_n y_n & y'_n & x_n & y_n & 1
\end{bmatrix}}_{\bm{A}} \bm{f} = \bm{0}.
\end{align}

The solution is then computed by finding the null space of matrix $\bm{A}$. When more than eight noisy matching pixels are provided as input, RANSAC is applied to identify reliable matches (inliers) to compute $\bm{F}$. Given the inlier set, the solution of $\bm{f}$ is then refined by calculating the corresponding right singular vector of $\bm{A}$ with the smallest singular value. This method is the well-known eight-point algorithm, where sparse feature matches are typically used. 

However, solving the null space of equation (\ref{eq:C4_fund_linear2}) only minimises the algebraic error $||{\bm{x}'}^T \bm{F} \bm{x}||$, which does not guarantees the minimisation of a meaningful geometrical distance, nor reflect any weighting of the inliers. One well-known method minimizes the Sampson distance~\cite{Torr98}\cite{Zhang98}, which modifies the rows of matrix $\bm{A}$ by a multiplicative scaling, such that
%\begin{multicols}{2}
\begin{align}
\label{eq:fund_linear_weighted}
\begin{bmatrix}
\phi_1 x'_1 x_1 & \phi_1 x'_1 y_1 & \phi_1 x'_1 &  \cdots & \phi_1 \\
\vdots & \vdots & \vdots & \ddots  & \vdots \\
\phi_n x'_n x_n & \phi_n x'_n y_n & \phi_n x'_n & \cdots & \phi_n 
\end{bmatrix} \bm{f} &= \bm{0}, 
\end{align}
%\end{multicols}
with
\begin{equation}
\label{eq:C4_Sampson_weight}
\phi_i = \frac{1}{\sqrt{({\bm{F}} \bm{x}_i)_1^2 + ({\bm{F}} \bm{x}_i)_2^2 + ({\bm{F}}^T \bm{x}'_i)_1^2 + ({\bm{F}}^T \bm{x}'_i)_2^2}},
\end{equation} 
where ${\bm{F}}$ is the iteratively refined Fundamental matrix that is first initialised by computing the null space of $\bm{A}$ from (\ref{eq:C4_fund_linear2}). The rank $2$ constraint is also enforced on the solution to obtain the final estimate of $\bm{F}$. 

In this work, we propose to use the uncertainty of the optical flow to accurately weigh each equation (row of matrix $\bm{A}$) during the refinement step of the Mahalanobis eight-point algorithm. It ensures that the Fundamental matrix solution minimizes the squared Mahalanobis distance to all the inlier correspondences with respect to the uncertainty. This process was illustrated in Fig.~\ref{fig:Mahalanobis_point2line} in the introduction. The weight $\{\phi\}$ can be computed from the minimum Mahalanobis distance in the figure.
% in (\ref{eq:8pointM})

\begin{theorem}
For the extimated information matrix, the new multiplicative scaling to each row in (\ref{eq:fund_linear_weighted}) becomes 
\begin{equation}
\label{eq:mahalEpi}
\phi_i = \sqrt{\frac{{Y}_{xx} {Y}_{yy} - {{Y}_{xy}}^2}{{({\bm{F}} \bm{x_i})_1}^2 {Y}_{yy} + {({\bm{F}} \bm{x_i})_2}^2 {Y}_{xx} - 2 {({\bm{F}} \bm{x_i})_1} {({\bm{F}} \bm{x_i})_2} {Y}_{xy}}},
\end{equation}
where the notation $(\bm{v})_k$ is the $k^{th}$ element of the vector $\bm{v}$. 
\end{theorem}

\begin{proof}
Given an initial Fundamental matrix estimate $\bm{F}$, homogeneous coordinates of matching pixels in both images $\bm{x}_i$ and $\bm{x}_i'$. We assume the error is only present in the second image's features $\bm{x}_i'$. From Fig.~\ref{fig:Mahalanobis_point2line}, let mean $\bm{\mu} = [x_0, y_0]^T$, information matrix $\bm{Y} = \bm{P}^{-1} = \begin{bmatrix}
{Y}_{xx}   & {Y}_{xy}  \\ 
{Y}_{xy} &  {Y}_{yy} 
\end{bmatrix}$, and a point on the line be $[x_1, y_1]^T = [x_1, \frac{-a x_1 - c}{b}]^T$. 

First, we calculate the minimum Mahalanobis distance between the line $\bm{l}$ and the mean image feature location $\bm{\mu}$. The minimum Mahalanobis distance equals the square root of the minimum squared Mahalanobis distance. The squared Mahalanobis distance ${d_M}^2$ between the feature pixel, and the epipolar line is computed as follows. 

\begin{equation}
\label{eq:SquaredMahalanobis}
{d_M}^2 =\begin{bmatrix}
x_1 - x_0 \\
\frac{-a x_1 - c}{b} - y_0
\end{bmatrix}^T \begin{bmatrix}
{Y}_{xx}   & {Y}_{xy}  \\ 
{Y}_{xy} &  {Y}_{yy} 
\end{bmatrix} \begin{bmatrix}
x_1 - x_0 \\
\frac{-a x_1 - c}{b} - y_0
\end{bmatrix}
\end{equation}
Expanding (\ref{eq:SquaredMahalanobis}) and computing the first derivative of ${d_M}^2$ with respect to $x_1$ equals to zero provides us the solution of $x_1$ where ${d_M}^2$ is minimum. We then substitute this solution of $x_1$ back into (\ref{eq:SquaredMahalanobis}) and apply a square root to obtain the equation of the minimum Mahalanobis distance, $\min(d_M)$ as follows. 

\begin{equation}
\label{eq:minMahalanobis}
\min(d_M) = |a x_0 + b y_0 + c| \sqrt{\frac{{Y}_{xx} {Y}_{yy} - {{Y}_{xy}}^2}{a^2 {Y}_{yy} + b^2 {Y}_{xx} - 2 a b {Y}_{xy}}}
\end{equation}

Since the original eight point algorithm minimises $|a x_0 + b y_0 + c|$, the multiplicative scaling is thus
\begin{equation}
\phi = \sqrt{\frac{{Y}_{xx} {Y}_{yy} - {{Y}_{xy}}^2}{a^2 {Y}_{yy} + b^2 {Y}_{xx} - 2 a b {Y}_{xy}}}.
\end{equation}
This completes the proof. 
\end{proof}

Similar to the Sampson distance method, ${\bm{F}}$ is then iteratively refined that is first initialized by computing the null space of $\bm{A}$ from (\ref{eq:C4_fund_linear2}). In the modified RANSAC step, the uncertainty of the optical flow is used to guide the sampling of the matches by increasing the likelihood of selecting correspondences (or matches) with lower uncertainty. The multinomial resampling method is used as in the particle filtering~\cite{Douc05}, which helps in decreasing the required number of iterations to ensure good inlier set selection. 

From the estimated Fundamental matrix $\bm{F}$ and the intrinsic camera matrix $\bm{K}$, the essential matrix $\bm{E}$ can be recovered as $\bm{E} = \bm{K}'^T \bm{F} \bm{K}$, and the camera pose $(\bm{t},\bm{R})$ can be obtained from the well known factorization method \cite{Hartley03}. 

%%%%%%%%%%%%%%%%%%%%%%%%%%%%%%%%%%%%%%%%%%%%%%
\section{Visual Processing Pipeline} \label{sec:pipeline}

%\subsection{Scale Estimation} \label{sec:scale}
Unlike stereo odometry, the estimated translation from the monocular camera faces problems with the scale ambiguity and the scale drift over multiple images.

Two methods are proposed to minimize the scale drift: fitting the ground plane and using the computed depth map. First, the scale is determined by fitting a plane through the $3$D reconstructed points roughly parallel to the camera axis's $zx$ (forward-right) plane. Assuming the ground is visible roughly in the middle of the image, we use the reconstructed points below the camera ($y$ coordinate of the $3$D points is positive), and not too far to the side (image coordinate $x$ within half the image width from the center) of the camera. Plane fitting provides us with a plane equation satisfying $ax+by+cz+d=0$. The plane's height with respect to $1$ unit of inter-frame translation then equals $-d/b$. If the camera's height, $h$ is known (calibrated from training data, or estimated throughout the motion), the scale of the inter-frame translation $s$ can be computed as $s = -(b h)/d$. Second, the translational scale is also computed by relating the current and previous depth maps. The median of the multiplicative factor between corresponding depth values provides a robust translational scale estimate. 

We can combine the scales estimated from the ground height and depth map for ground-based vehicles using a simple average. For aerial vehicles, we cannot ensure that the height from the ground is fixed, and thus, we estimate the scale using ground height for the first frame and relies on the scale from the depth map for subsequent frames. The camera's height is continually updated using the reconstructed 3D points of the scene, which is only used to reinitialize the translational scale when not enough ($<5\%$) $3$D points from the previous estimate overlaps the current triangulated points. The estimated scale is then multiplied to the estimated camera translation, 3D reconstructed points, depth map, and uncertainty. The information matrix of the reconstructed 3D points is divided by the squared scale. 

%\subsection{Inter-frame Pose} \label{sec:pose_fuse}
Given the dense optical flow, there are two methods to estimate the inter-frame poses depending on the magnitude of the pixel motion (or parallax). If there exists enough parallax, the Mahalanobis eight-point algorithm (Section~\ref{sec:Weighted8Pt}) along with the estimated scale can be applied. If the motion is too small (e.g., hovering), the PnP method~\cite{Gao03} becomes more effective to estimate the motion of the camera. The poses computed from two methods are averaged to improve the robustness of the solution. When the depth values have not converged accurately, the PnP estimate may return an error-prone result. Thus, we only perform the fusion when the difference in the estimated translation scale is within $30\%$ of the scale estimated in Section~\ref{sec:scale}, and the estimated rotations have a difference less than $0.5$ radians. If either of these conditions is not met, we use the pose calculated from the Mahalanobis eight-point algorithm (Section~\ref{sec:Weighted8Pt}) instead. The depth maps are fused using the pixels that are matched and triangulated. The fused depth map and reconstructed scene points for the previous frame are propagated to the current image frame using the computed pose. This map provides prior $3$D scene information for the next image frame.

%\subsection{Small Motion Handling}
The use of scaled-down images (one-third of the original scale) for dense optical flow estimation cannot guarantee the accuracy of the matches when the pixel translation between two images is too small. This case occurs when the vehicle moves very slowly or hovers, causing the motion estimation to be error-prone. Small translational motion estimation is a common problem in most monocular visual odometry, as the small parallax between two images leads to difficulty in estimating both motion and structure accurately. Two conditions determine the sufficient parallax. First, the Shi-Tomasi corner matches have a median displacement magnitude of at least $2.5$ pixels. Secondly, the third quantile ($75$\%) of the computed optical flow is larger than $5$ pixels. Suppose either of the two conditions is not met. In that case, the inter-frame motion is computed only from the PnP method using the previously calculated depth and the motion of the corresponding pixels. 

%\subsection{Loop-Closures} \label{sec:C4_LC}
Loop closure (or revisiting the previous locations) detection is accomplished by using the robust linear pose-graph optimization method from \cite{Cheng15}. Similar to other pose-graph SLAM, this method treats all poses of the vehicle or robot as vertices and the inter-pose constraints (\emph{e.g.} odometry and close loop constraints) as edges. The candidate frames for loop closure are found in three steps. The first step is by selecting frames with their estimated poses to be within a fixed distance while having a difference in frame index no less than a threshold value. The minimum frame index difference is enforced to prevent finding too many candidates within adjacent frames. We can further reduce the possible candidates by only finding candidate loop closure images for every 10 frames. The second step is to determine which of the candidate loop closure images are valid by using the structural similarity index (SSIM)~\cite{Wang04}. Any images are discarded that have an SSIM index less than a set threshold (experimentally set to $0.38$), and keep a maximum of three candidate images with the highest SSIM value. Lastly, the dense optical flow between the images and their possible neighbors are computed. The estimated uncertainty is used to determine if the optical flow is reliable and only calculate their inter-frame motion when the percentage of matches with an uncertainty less than a set threshold is higher than $20$\% (an example is shown in Fig.~\ref{fig:OF_uncertainty}). During the loop closure, the motion estimation step also checks for the small-motion conditions discussed in the previous section, which enables loop-closures between temporally separated poses but spatially close to each other. %TODO: uncertainty

%%%%%%%%%%%%%%%%%%%%%%%%%%%%%%%%%%%%%%%%%%%%%%%%%
\section{Experimental Results} \label{sec:Experiment}

We evaluated our proposed SLAM framework using the public KITTI dataset \cite{Geiger12} and our UAV dataset. KITTI dataset shows a camera mounted on a vehicle travelling on the roughly planar ground. Sequence 01, in particular, is a challenging highway scenario, where the car is moving at high speed, and there are few unique feature points within view. UAV dataset shows a camera mounted on a quadcopter flying in a highly unstructured outdoor environment with dynamically moving objects. The UAV also performs motions such as (almost) pure rotation and extreme height variation. These make an accurate estimation of camera pose difficult for existing monocular visual odometry and SLAM. 

%\subsection{Triangulation Simulation}
\subsection{Ground Vehicle}
The dataset we used to verify the performance of our proposed algorithm for ground-based vehicles is taken from the KITTI benchmark. For optical flow evaluation, we use the flow 2015 dataset~\cite{Menze15CVPR}, while for the odometry, we use the odometry dataset~\cite{Geiger12}. For both experiments, we use the monocular RGB images (\emph{image\_2} folder). In the odometry experiment, we assume the camera is $1.7 m$ above the ground, with zero pitch. Due to the post-processing part of the DCFlow code not being made available, we can only verify the optical flow result before homography fitting is applied to the EpicFlow~\cite{Revaud15} interpolated results. Based on KITTI 2015 optical flow dataset, by utilizing our epipolar constraint on the cost volume, we achieved a $0.6\%$ improvement in accuracy (regarding less than 3 pixels endpoint error criterion). The improvement is small because the epipolar truncation cost is set very low to accommodate for dynamic pixels in the scene. However, we can visually observe a noticeable improvement in the optical flow estimation for the ground pixels, not reflected by the significant ($3$ pixels error) KITTI accuracy metric. We also implemented a homography fitting step based on the description of their paper. 

The uncertainty estimate for the dense optical flow is visually inspected, where it was observed that occluded, out-of-bound or textureless regions of the image have high uncertainty value. For ground-based vehicle's visual odometry result, we compare our performance with existing methods, specifically VISO2-M~\cite{Geiger11}, MLM-SFM~\cite{Song16}, PMO~\cite{Fanani17} and DOF-1DU+LC~\cite{Ng17}. We selected a few of the available sequences that contain a slow-moving vehicle in an urban environment (sequence 00), a fast-moving vehicle on a highway (sequence 01) and a vehicle travelling in a loop (sequence 06) to gauge the performance of our proposed methods. The results are summarised in Table~\ref{tab:odomRes_ablation} and Table~\ref{tab:odomRes}. Note that VISO2-M~\cite{Geiger11} and MLM-SFM~\cite{Song16} methods fail to estimate the visual odometry for sequence 01 due to the highly repeated structures of the scene, which cannot be reliably matched by the sparse feature matching technique their methods employ. Fig.~\ref{fig:estimated_pose} shows our estimated trajectory for the vehicle's motion. From the estimated motion trajectory (Fig.~\ref{fig:estimated_pose}) and computed error from the ground truth (Table~\ref{tab:odomRes}), we can observe that our proposed method achieved a very accurate estimation of translation. This result is achieved without using bundle adjustment, motion model, or ground segmentation used by other state-of-the-art techniques. From Table~\ref{tab:odomRes_ablation}, we can also observe an improvement in the rotation estimate after fusing the Mahalanobis eight-point algorithm and PnP result.

\begin{table*}[th]
\centering
\begin{tabular}{|c|cc|cc|cc|}
\hline
 & \multicolumn{2}{|c|}{DOF-2DU} & \multicolumn{2}{|c|}{DOF-2DU+PnP} & \multicolumn{2}{|c|}{DOF-2DU+PnP+LC} \\
 \cline{2-7}
seq & rot & trans & rot & trans & rot & trans \\
 & (deg/m) & (\%) & (deg/m) & (\%) & (deg/m) & (\%) \\
\hline
00 & 0.0076 & 1.80 & 0.0067 & 1.57 & 0.0045 & 1.07 \\
\hline
01 & 0.0082 & 0.97 & 0.0050 & 1.03 & 0.0050 & 1.03 \\
\hline
06 & 0.0047 & 0.96 & 0.0039 & 1.11 & 0.0039 & 1.17 \\
\hline
\end{tabular}
\vspace{0.1cm}
\caption{Ablation study of our new proposed methods for selected KITTI dataset. ``DOF-2DU'' is the pose estimate of our Mahalanobis eight-point algorithm using dense optical flow with 2-dimensional uncertainty, ``+PnP'' is the fused pose estimate with perspective-n-point, and ``+LC'' is the inclusion of loop closure. }
\label{tab:odomRes_ablation}
\end{table*}

\begin{table*}[th]
\centering
\begin{tabular}{|c|cc|cc|cc|cc|cc|}
\hline
 & \multicolumn{2}{|c|}{VISO2-M} & \multicolumn{2}{|c|}{MLM-SFM} & \multicolumn{2}{|c|}{PMO} & \multicolumn{2}{|c|}{DOF-1DU+LC} & \multicolumn{2}{|c|}{DOF-2DU+PnP+LC} \\
 \cline{2-11}
seq & rot & trans & rot & trans & rot & trans & rot & trans & rot & trans \\
 & (deg/m) & (\%) & (deg/m) & (\%) & (deg/m) & (\%) & (deg/m) & (\%) & (deg/m) & (\%) \\
\hline
00 & 0.0209 & 11.91 & 0.0048 & 2.04 & 0.0042 & 1.09 & 0.0117 & 2.03 & 0.0045 & 1.07 \\
\hline
01 & n/a & n/a & n/a & n/a & 0.0038 & 1.32 & 0.0107 & 1.149 & 0.0050 & 1.03 \\
\hline
06 & 0.0157 & 4.74 & 0.0081 & 2.09 & 0.0044 & 1.31 & 0.0054 & 1.05 & 0.0039 & 1.17 \\
\hline
\end{tabular}
\vspace{0.1cm}
\caption{Comparison of visual odometry accuracy for VISO2-M \protect\cite{Geiger11}, MLM-SFM \protect\cite{Song16}, PMO \protect\cite{Fanani17}, dense optical flow with 1D uncertainty and loop closure (DOF-1DU+LC)~\protect\cite{Ng17} and our new proposed methods (DOF-2DU+PnP+LC) for selected KITTI dataset. }
\label{tab:odomRes}
\end{table*}

\begin{figure*}[th]
\centering
	{\includegraphics[width=0.32\textwidth]{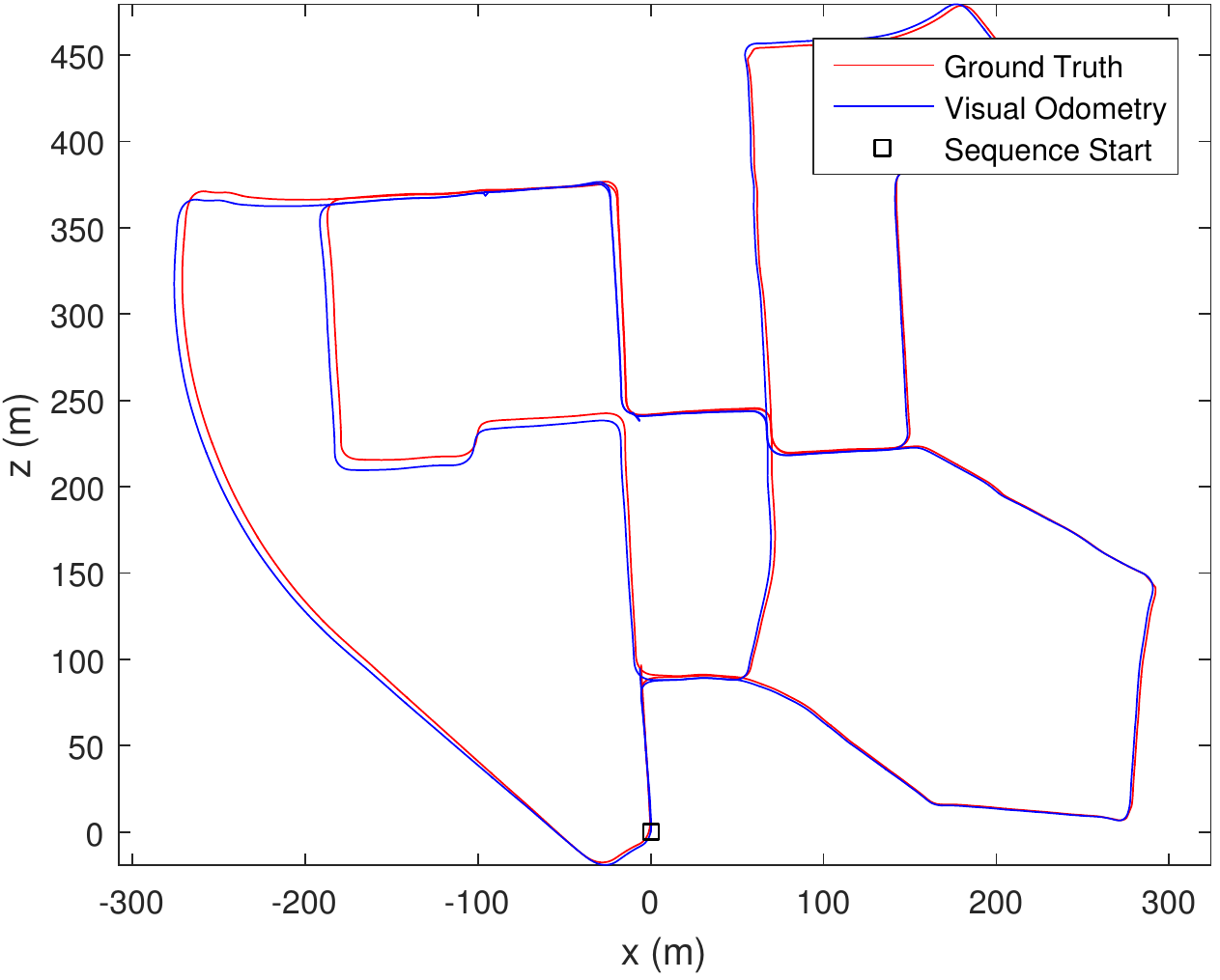}}
	{\includegraphics[width=0.32\textwidth]{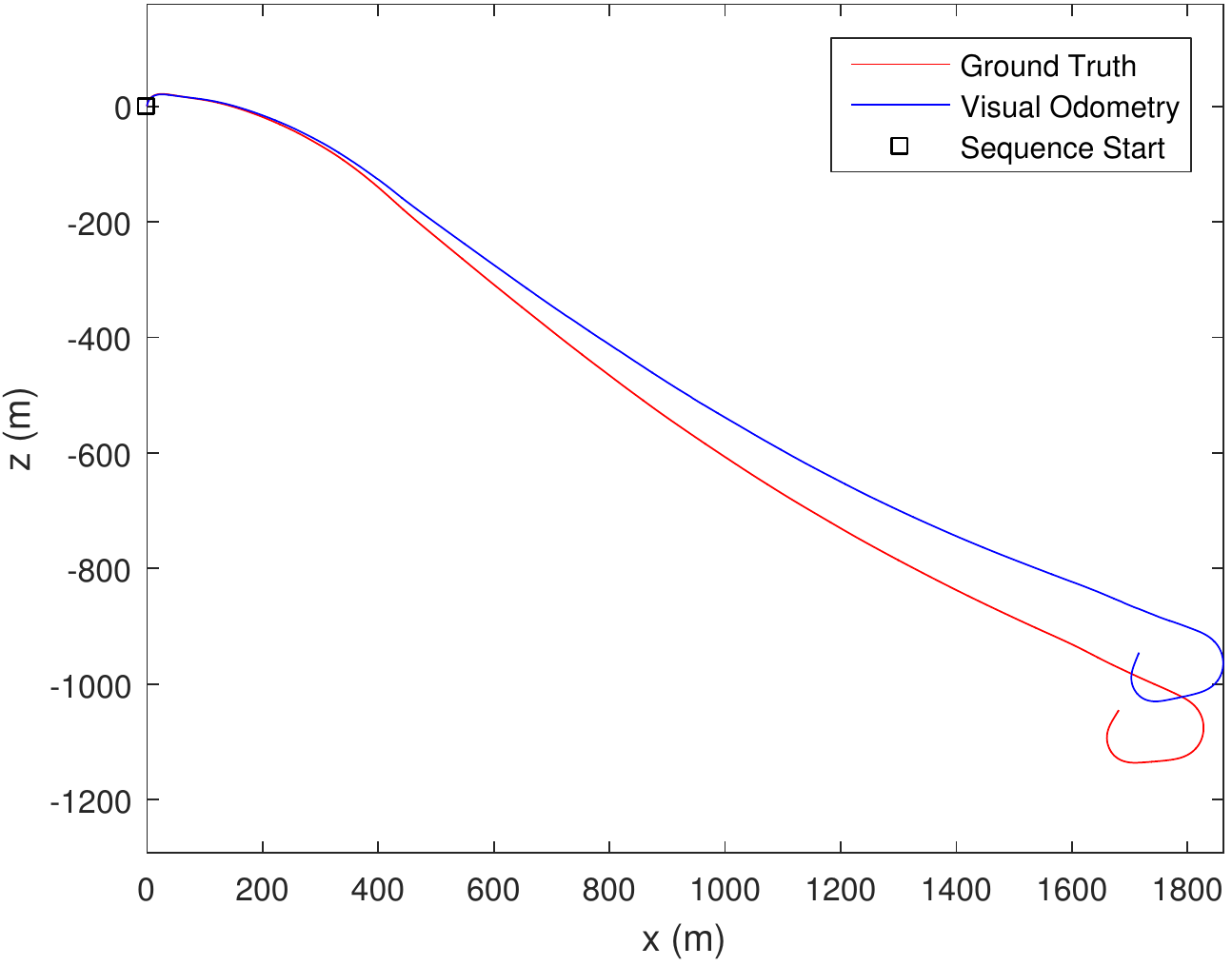}}
	{\includegraphics[width=0.32\textwidth]{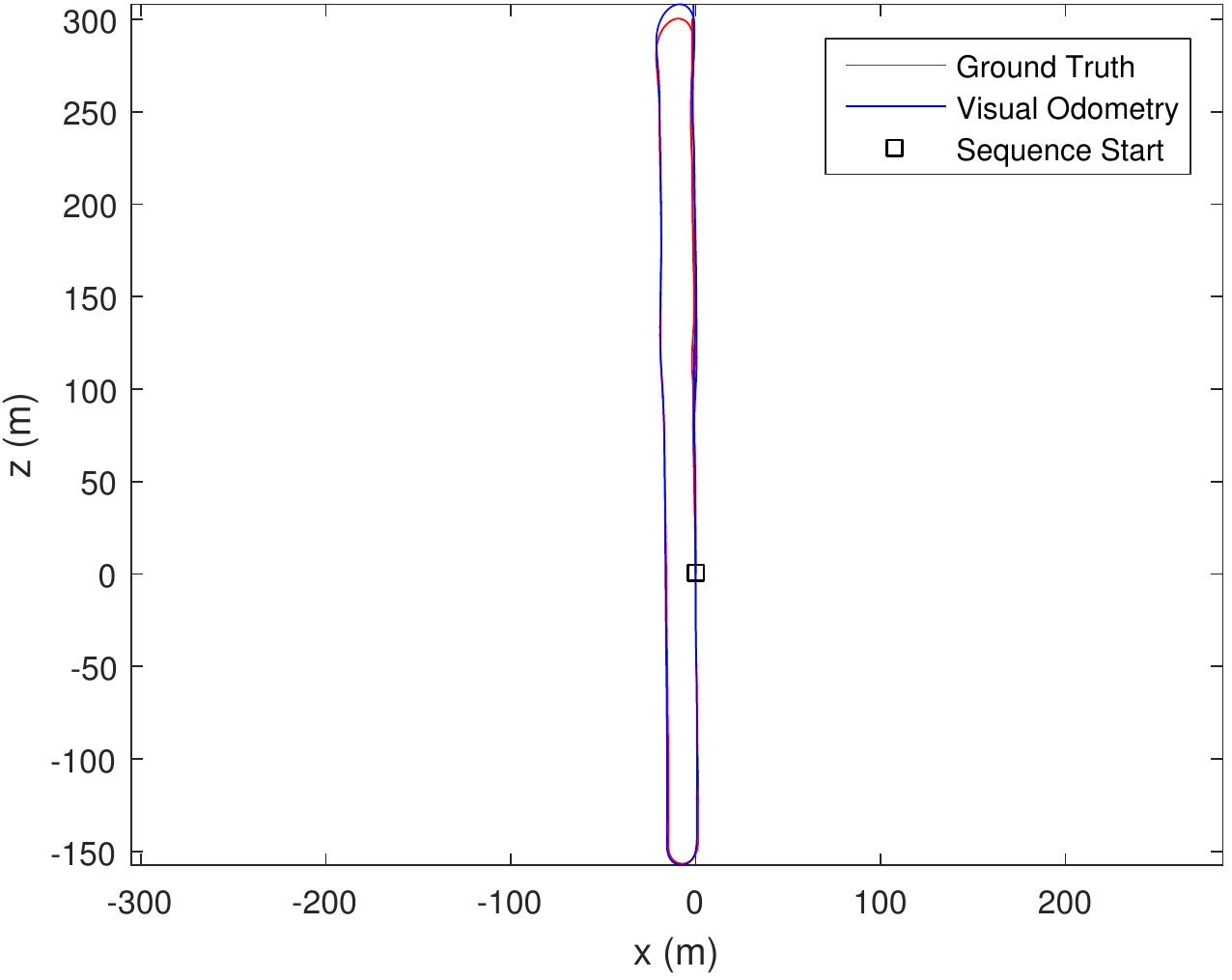}}
	\caption{\label{fig:estimated_pose} Comparison of the estimated motion trajectory and the ground truth motion. From left to right: (a) sequence 00, (b) sequence 01, (c) sequence 06. }
\end{figure*}

\subsection{Aerial Vehicle}
Since our visual odometry method does not rely on the restrictive motion model of the vehicle, we can easily apply our proposed method with slight modification to aerial vehicles (\emph{e.g.} UAV). The difference with a ground-based vehicle is that the camera height is not assumed constant but is updated for each motion since the aerial vehicle can change its height arbitrarily. Another challenge of aerial visual odometry is that it can rotate its yaw with no translation, which makes the pose estimation and 3D scene reconstruction highly under-constraint and error-prone. We also incorporated such motion in the video sequences we used in our experiment. 

%\subsubsection{Small Translation with Rotation}
We captured $500$ frames of video from a quadcopter flying among some trees near a road, where the scene has highly repetitive, unstructured and dynamic objects (\emph{e.g.} leaves, cars). Due to the lack of ground truth, unlike the KITTI dataset, we evaluate the scale drift by reversing the frames and appended them to the end of the video, where the last frame coincides with the first frame. Fig.~\ref{fig:aerial_exp} shows the result of our experiment. 

\begin{figure*}[th]
\centering
	{\includegraphics[width=0.32\textwidth]{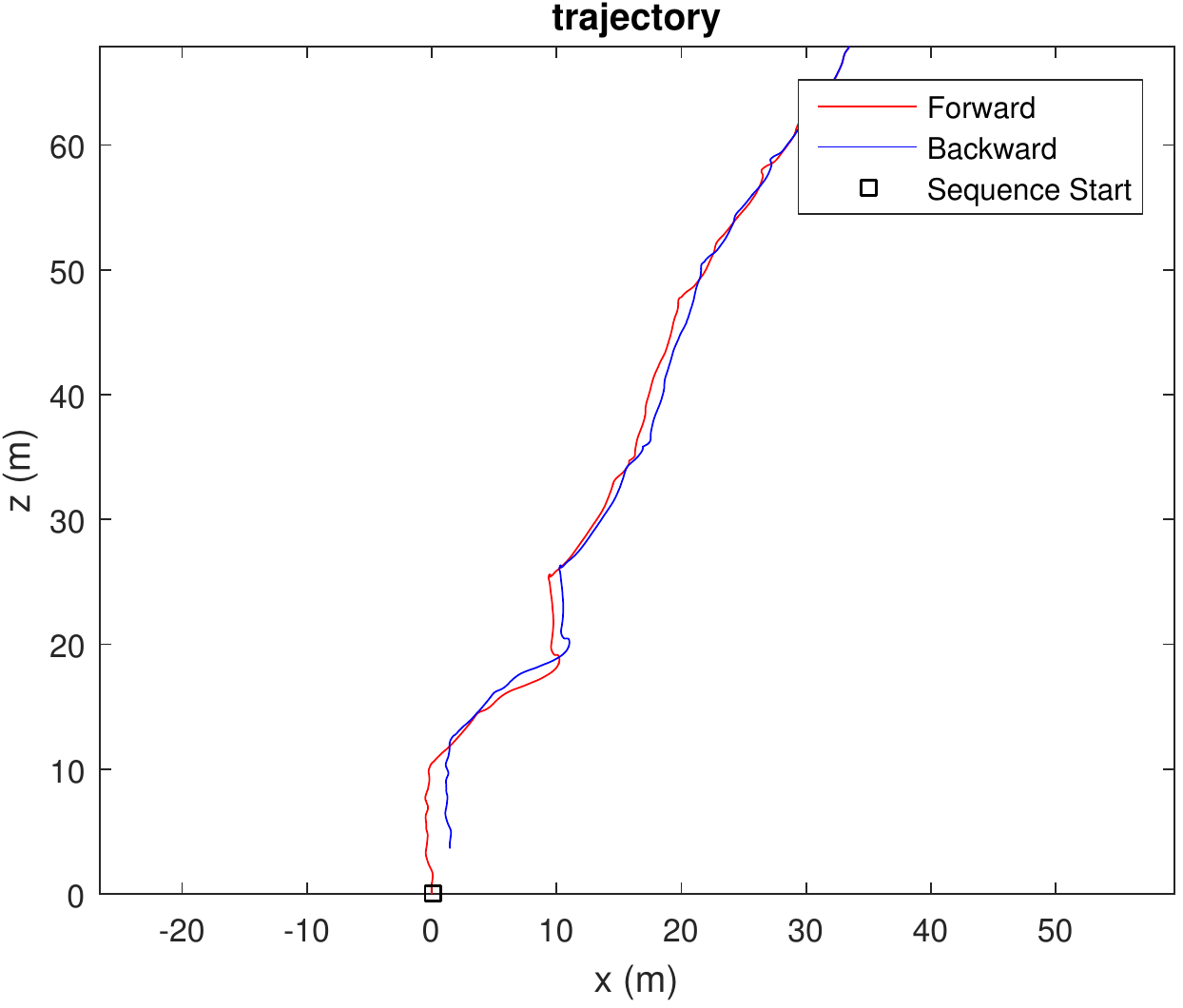}}
	{\includegraphics[width=0.325\textwidth]{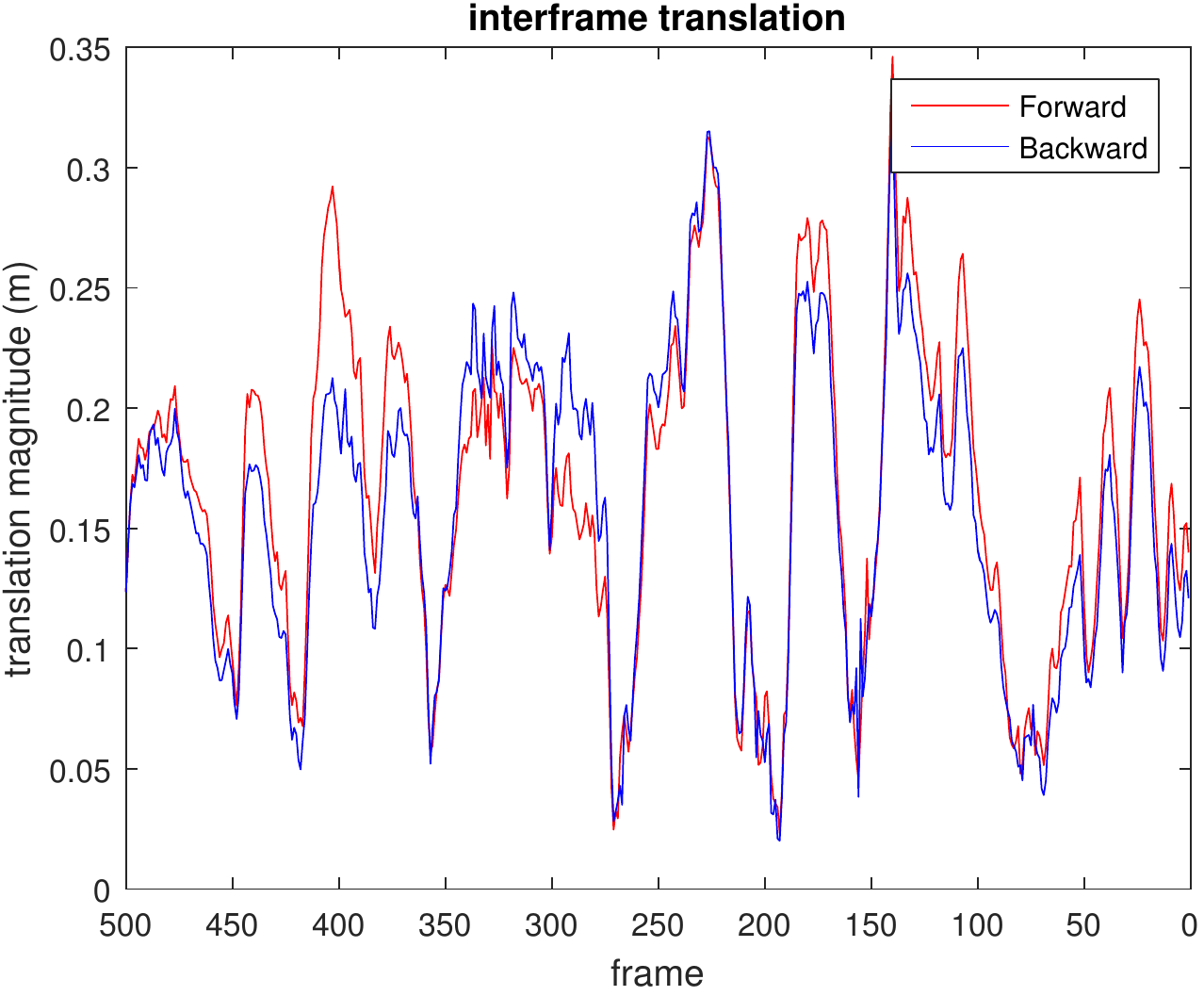}}
	{\includegraphics[width=0.315\textwidth]{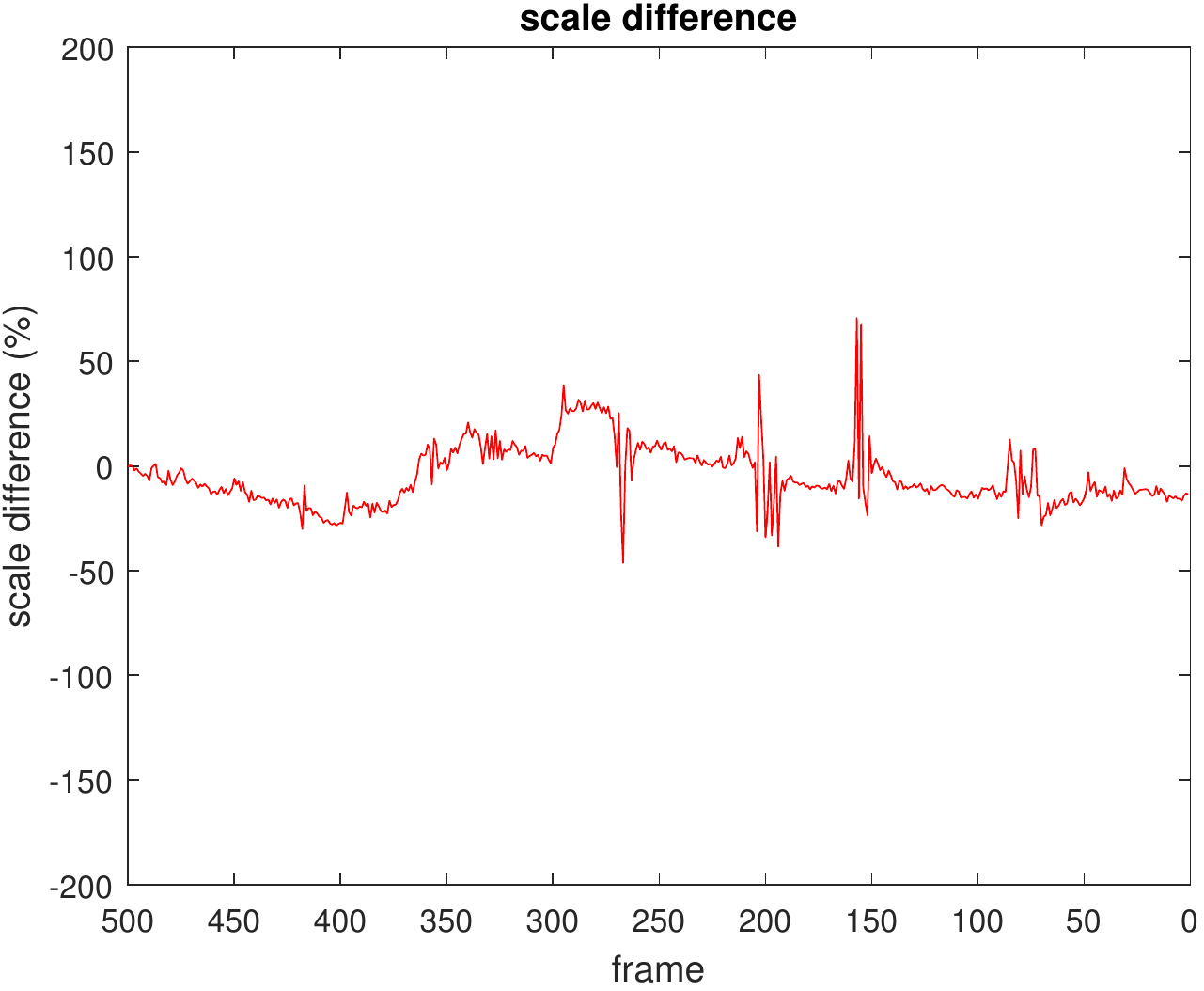}}
	\caption{\label{fig:aerial_exp} Plots evaluating the scale drift of our proposed visual odometry on UAV video. From left to right: (a) estimated motion trajectory, (b) inter-frame translation magnitude, (c) percentage scale difference (difference between the translation magnitude divided by forward magnitude). }
\end{figure*}

\begin{remark}
We do not compute the loop closure constraint for this video sequence. The result shown in Fig.~\ref{fig:aerial_exp} is pure visual odometry. 
\end{remark}

From the third plot of Fig.~\ref{fig:aerial_exp}, we can see that the translation scale difference remains close to zero, which shows that the scale drift is small. We also observed sudden spikes in the third plot, which corresponds to small motion, as can be seen from the central plot of Fig.~\ref{fig:aerial_exp}. As a comparison, we also evaluated VISO2-M~\cite{Geiger11} method on the same UAV video, using the constant camera height assumption. Fig.~\ref{fig:aerial_exp_libviso} shows the result. We observed that the estimated pose has a considerable translational magnitude (wrong) when the quadcopter rotates the yaw with negligible translational motion (\emph{e.g.} at frames 200, 150 and 100). From the third plot of Fig.~\ref{fig:aerial_exp_libviso}, we can also see that although the estimated scale does not drift (due to fixed camera height assumption), the estimated translational magnitude fluctuates erratically throughout the video sequence. 

\begin{figure*}[th]
\centering
    {\includegraphics[width=0.32\textwidth]{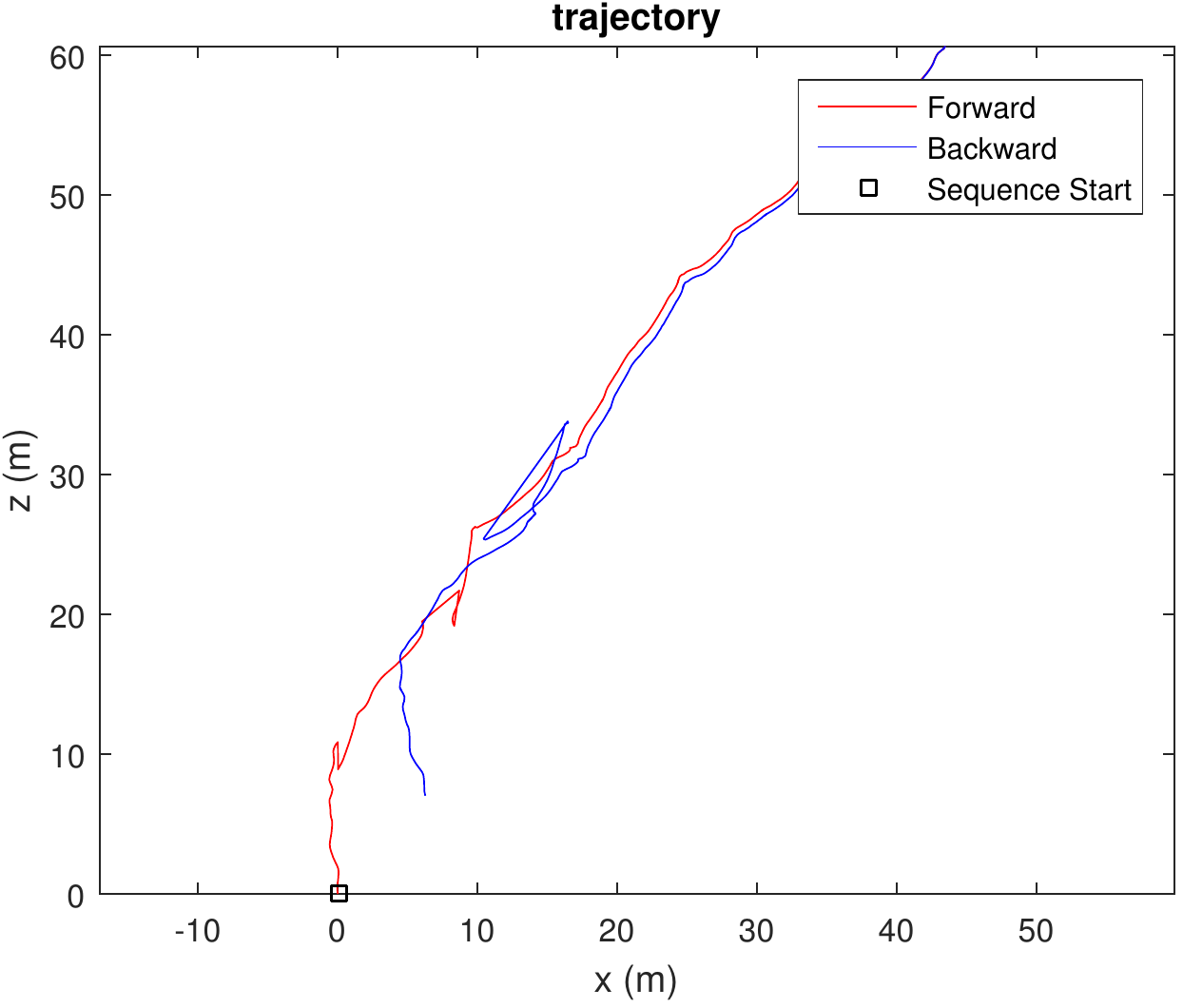}}
    {\includegraphics[width=0.325\textwidth]{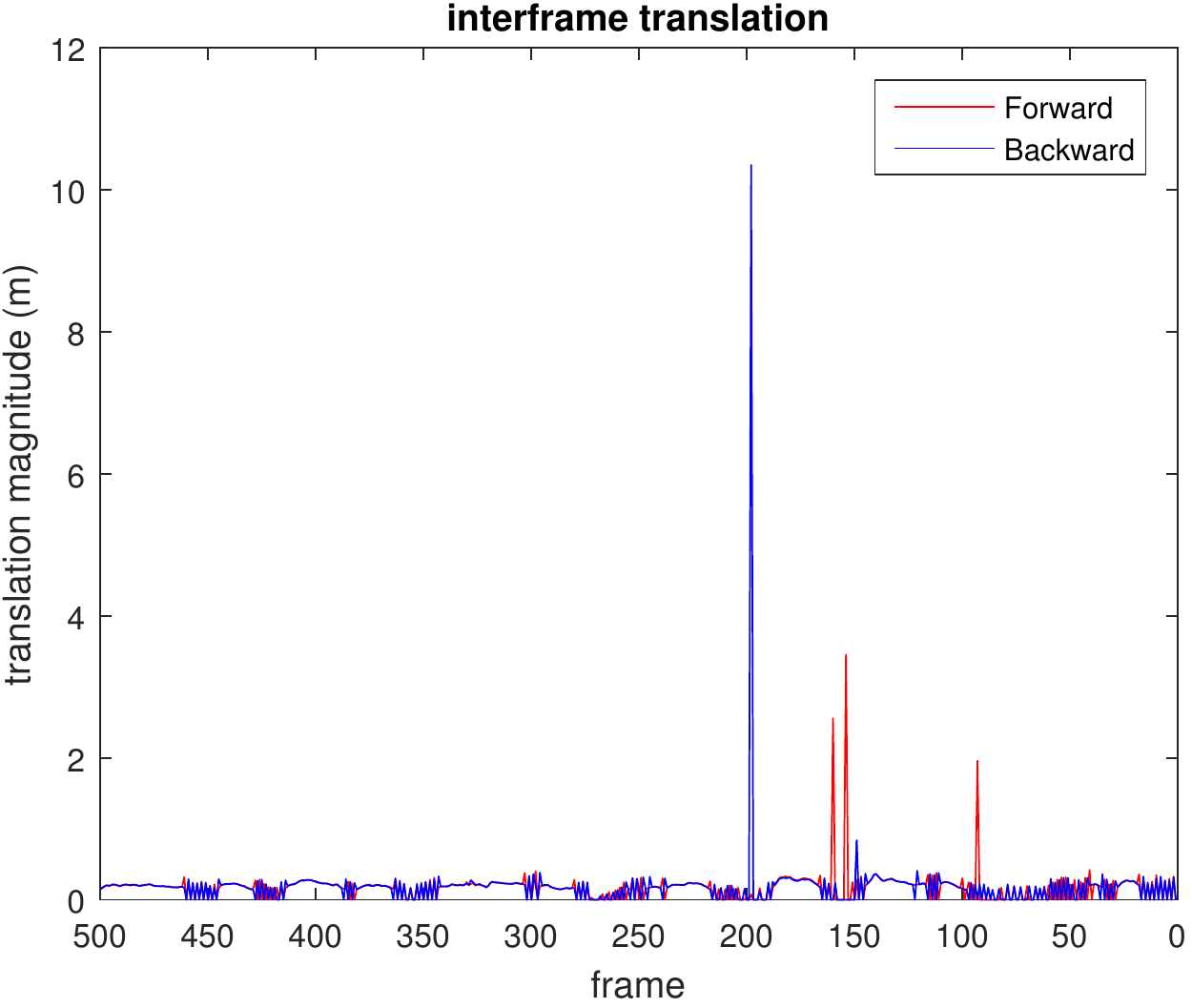}}
    {\includegraphics[width=0.315\textwidth]{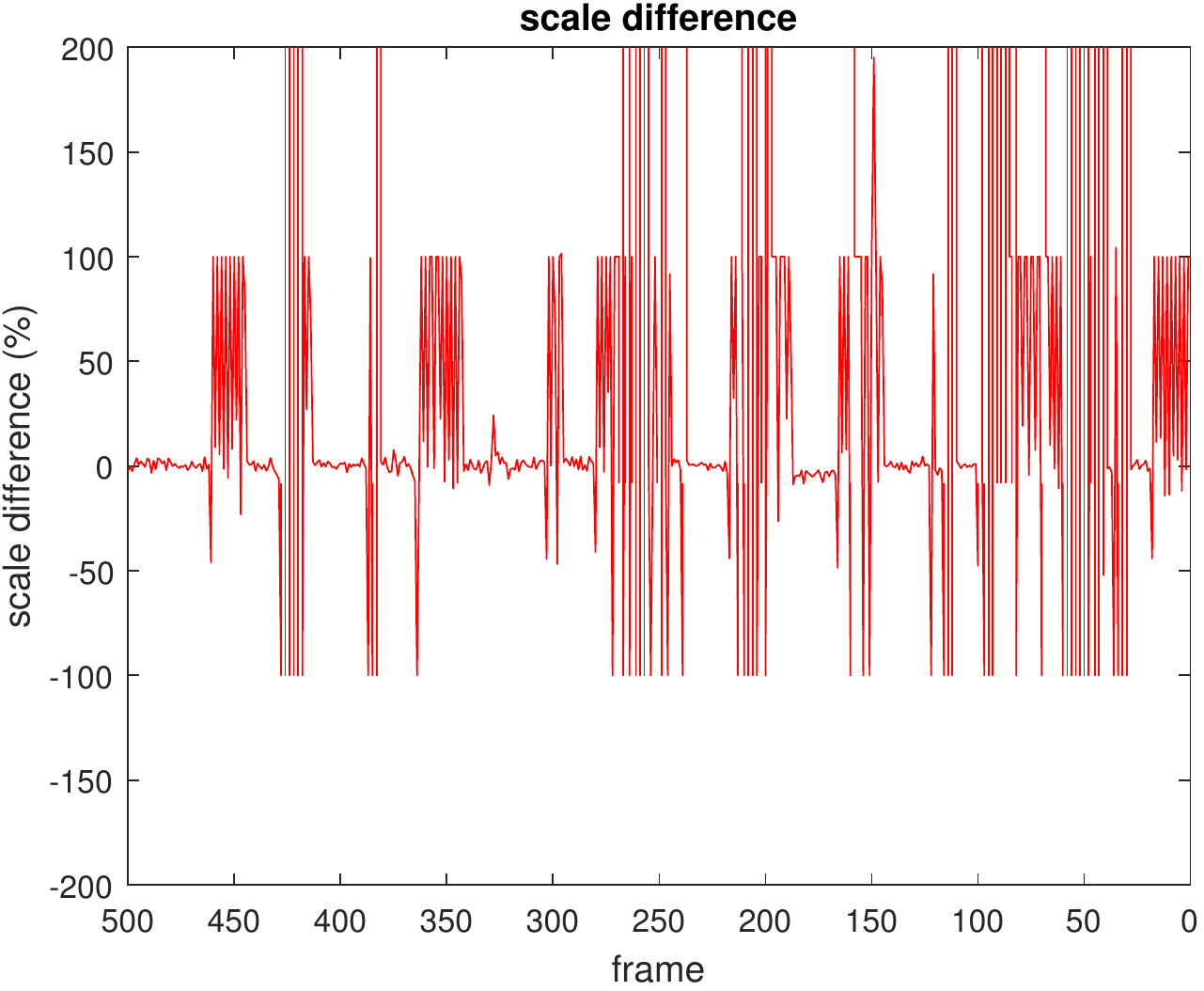}}
    \caption{\label{fig:aerial_exp_libviso} Plots evaluating the scale drift of our proposed visual odometry on UAV video. From left to right: (a) estimated motion trajectory, (b) inter-frame translation magnitude, (c) percentage scale difference (difference between the translation magnitude divided by forward magnitude). }
\end{figure*}

%\subsubsection{Fast Motion With Drastic Height Changes}
For the next experiment with UAV video, we captured 563 frames of a UAV flying at high speed with significant variation in height. We have also marked some trees with yellow tapes ($1 m$ apart) to calibrate the first translational scale and obtain a measure of scale drift after the UAV returns to the same spot. The error in the estimated position can also be visually observed by comparing the location of the reconstructed scene points. We plot 3D scene points with a depth standard deviation less than $0.2m$ for the first and last frames. Fig.~\ref{fig:aerial_fast_depth} and Fig.~\ref{fig:aerial_fast} shows our result. From Fig.~\ref{fig:aerial_fast}(c), we can see that the error of the estimated camera pose and reconstructed 3D scene points is minimal. The scale drift computed from the known distance between the tape is $+5.36 \%$. We have also calculated the distance between the farthest point from the starting location, compared to GPS measurement and VISO2-M result. The result in Table~\ref{tab:farthest} shows that our method agrees with GPS measurement more closely compared to the VISO2-M method. Thus, this verifies that our method can accurately estimate the camera motion, regardless of the motion dynamics of the vehicle or scene structure. 

\begin{figure*}[th]
\centering
	{\includegraphics[width=0.302\textwidth]{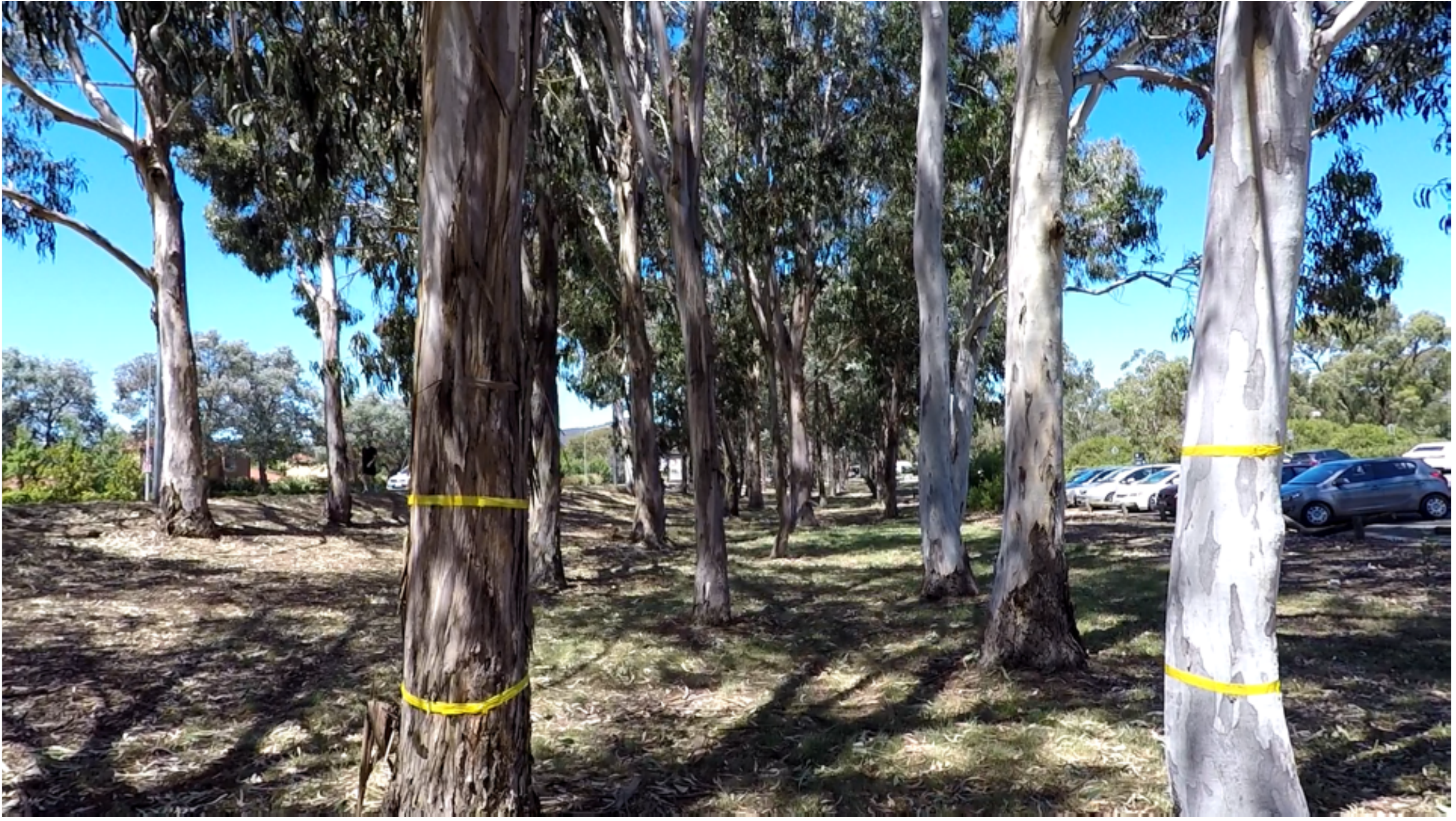}}
	{\includegraphics[width=0.32\textwidth]{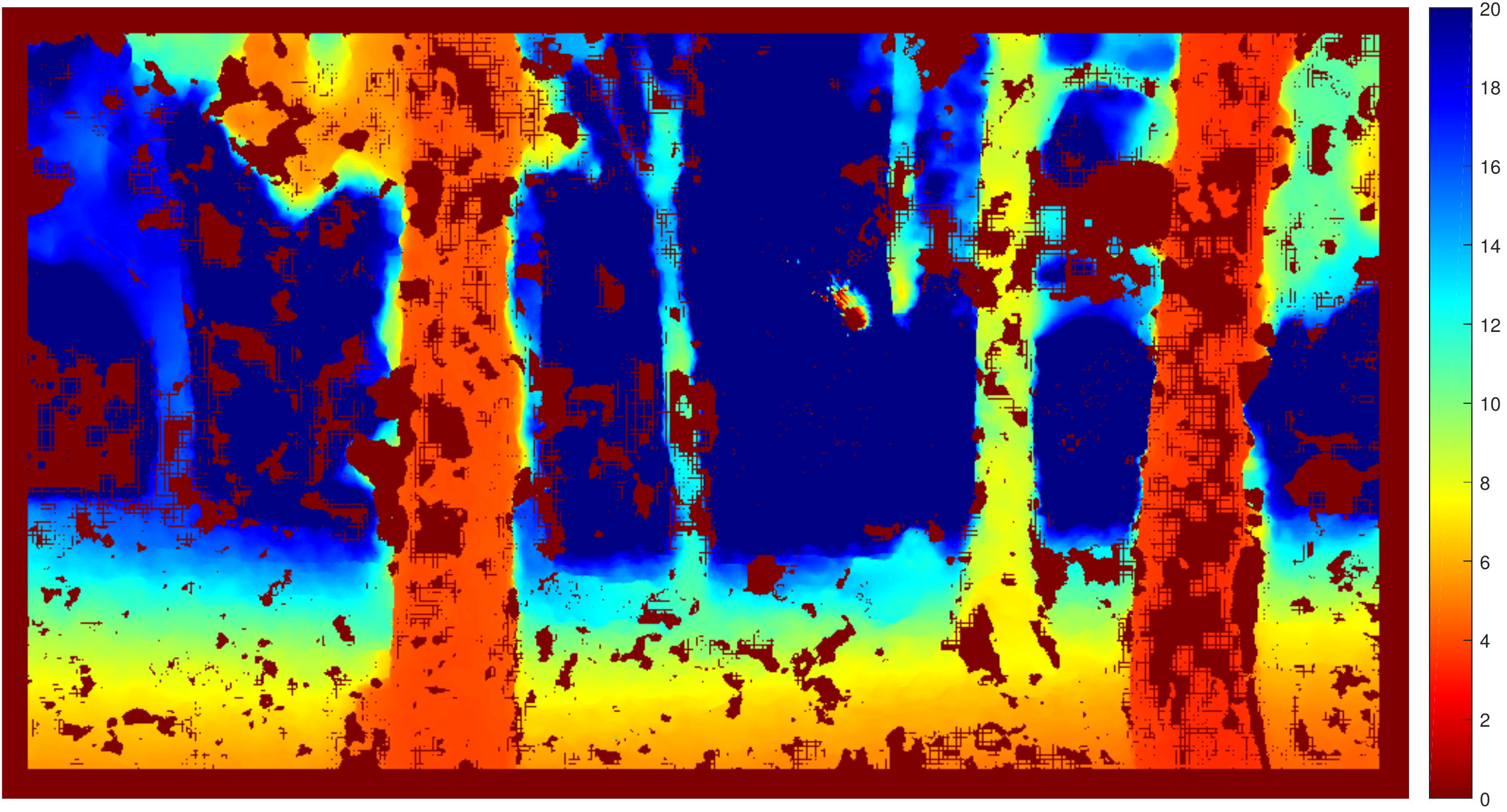}}
	{\includegraphics[width=0.32\textwidth]{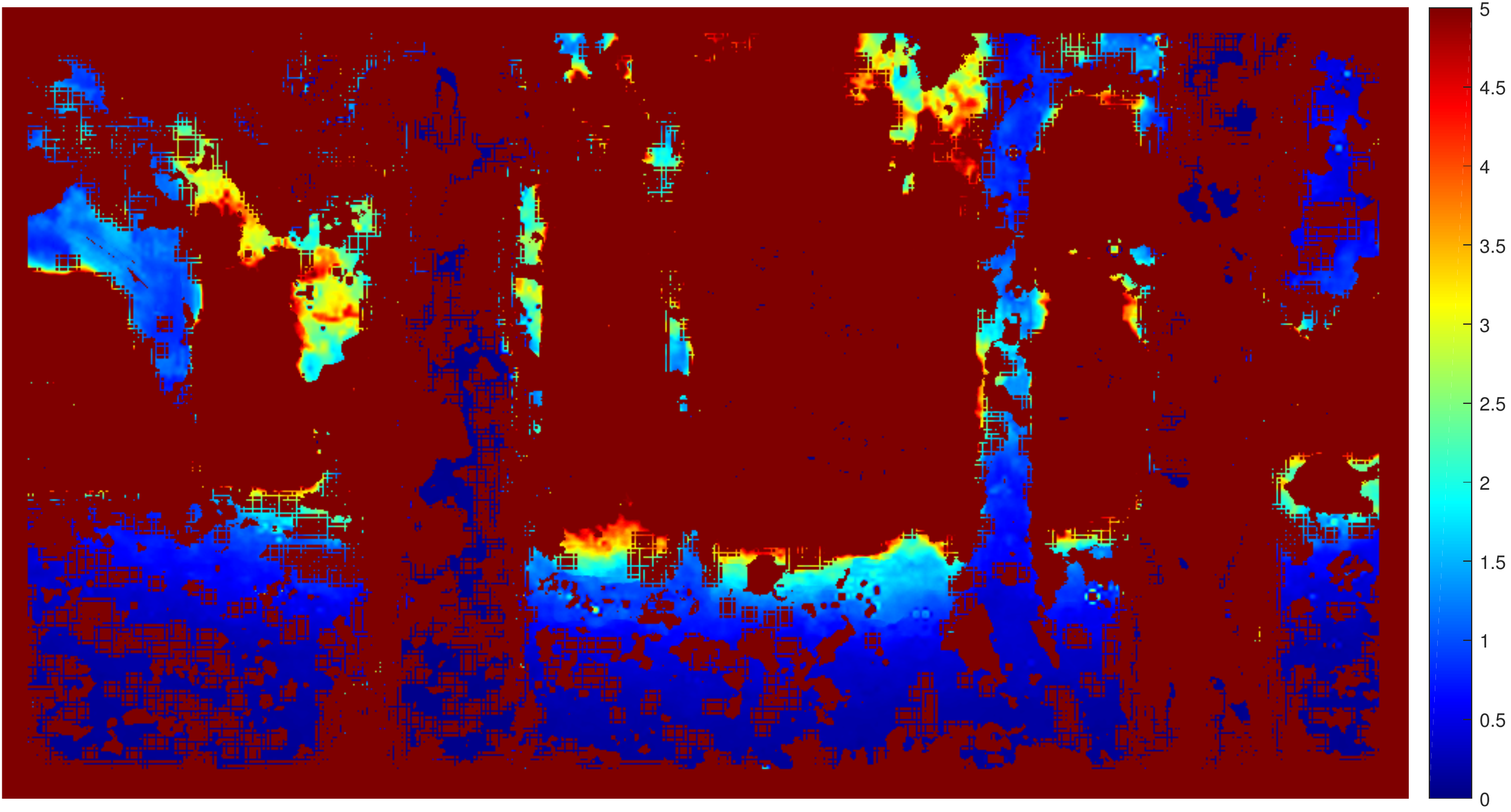}}
	{\includegraphics[width=0.302\textwidth]{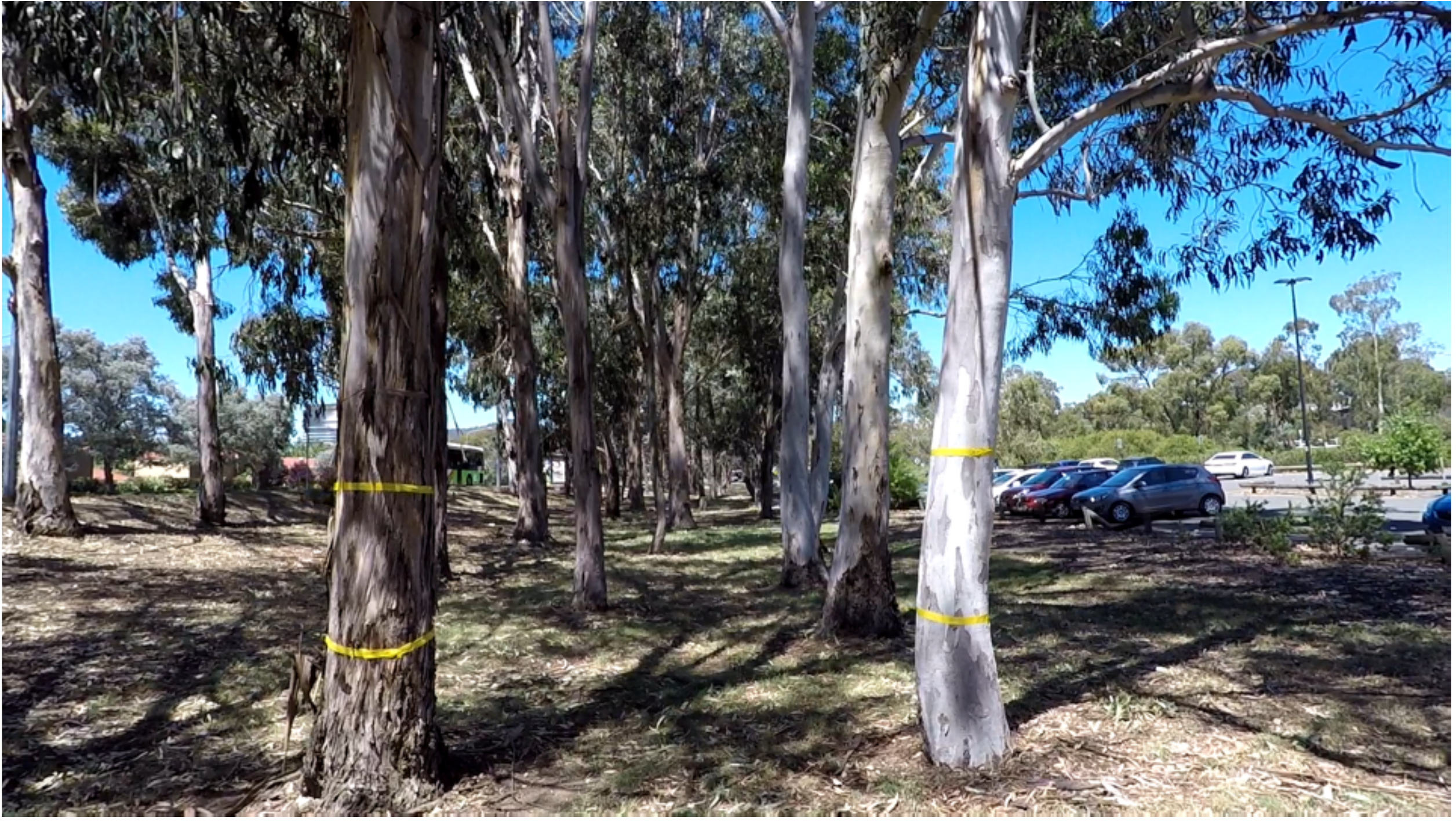}}
	{\includegraphics[width=0.32\textwidth]{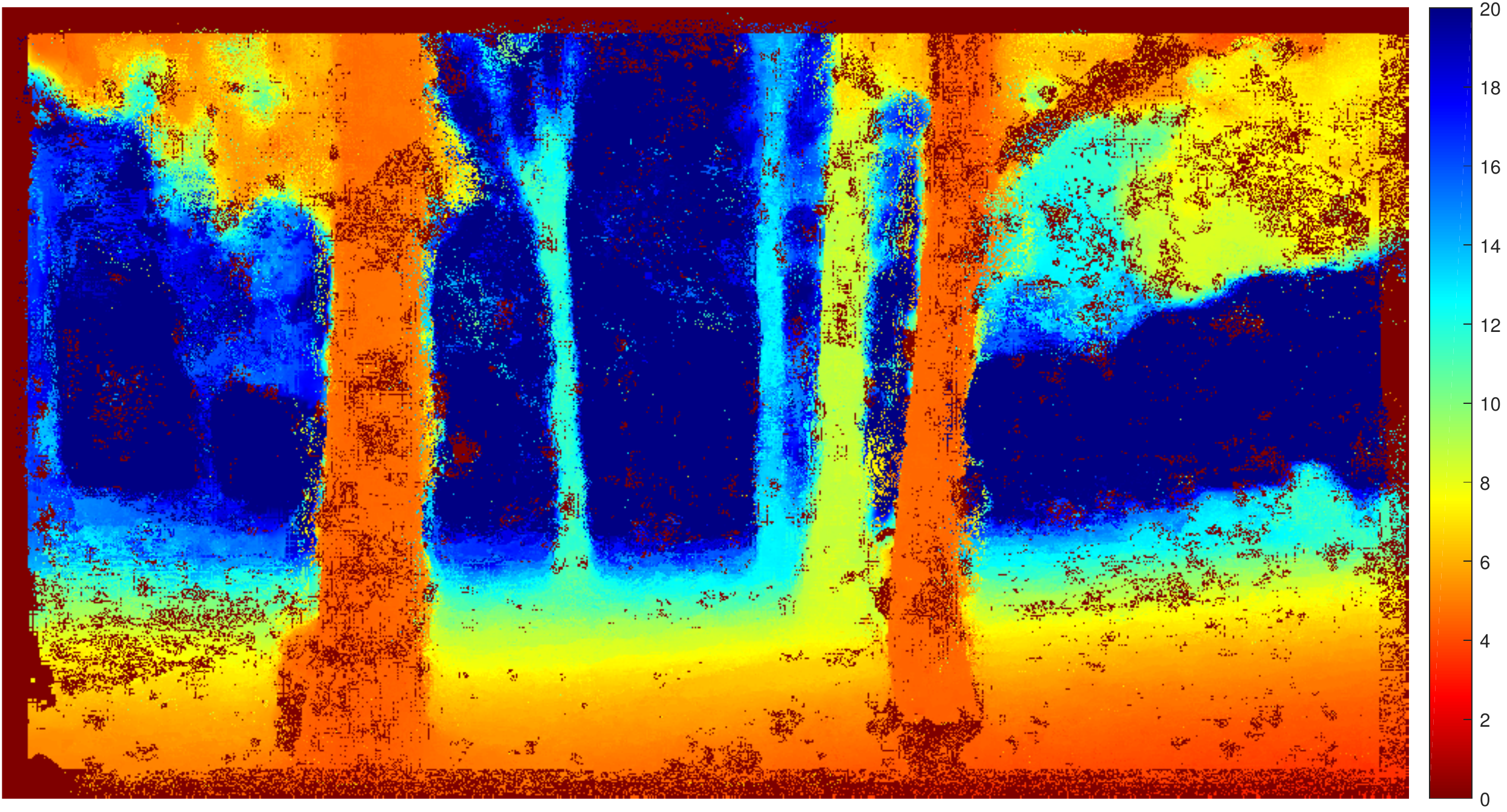}}
	{\includegraphics[width=0.32\textwidth]{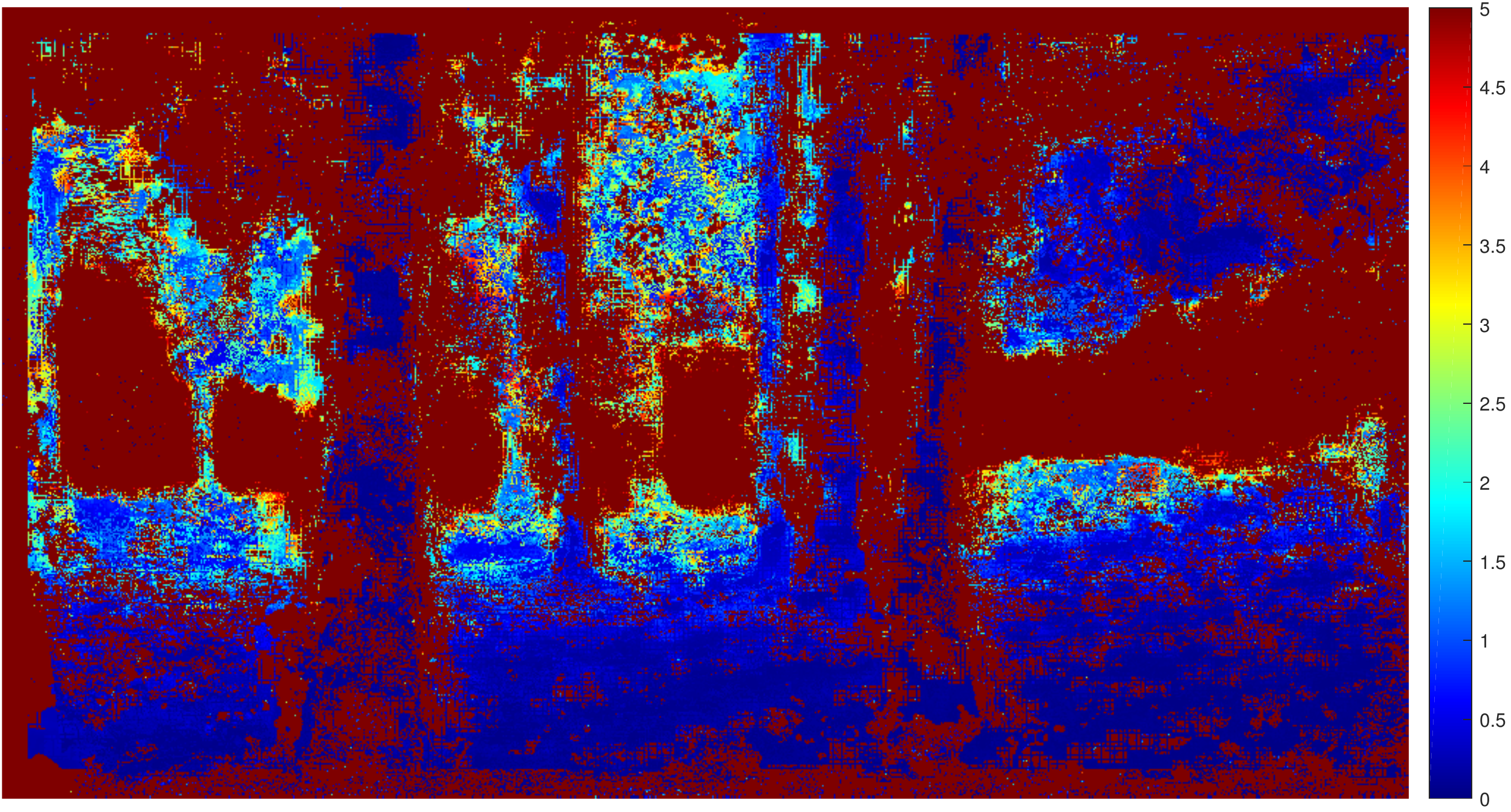}}
	\caption{\label{fig:aerial_fast_depth} Estimated depth with standard deviation. From left to right: input frame, estimated depth, estimated depth standard deviation. The first row is frame 0, second row is frame 562. The scale of the colour code is in meters. Pixels that are identified as outliers are not triangulated and appears dark red in the middle plot. }
\end{figure*}

\begin{figure*}[th]
\centering
	{\includegraphics[width=0.32\textwidth]{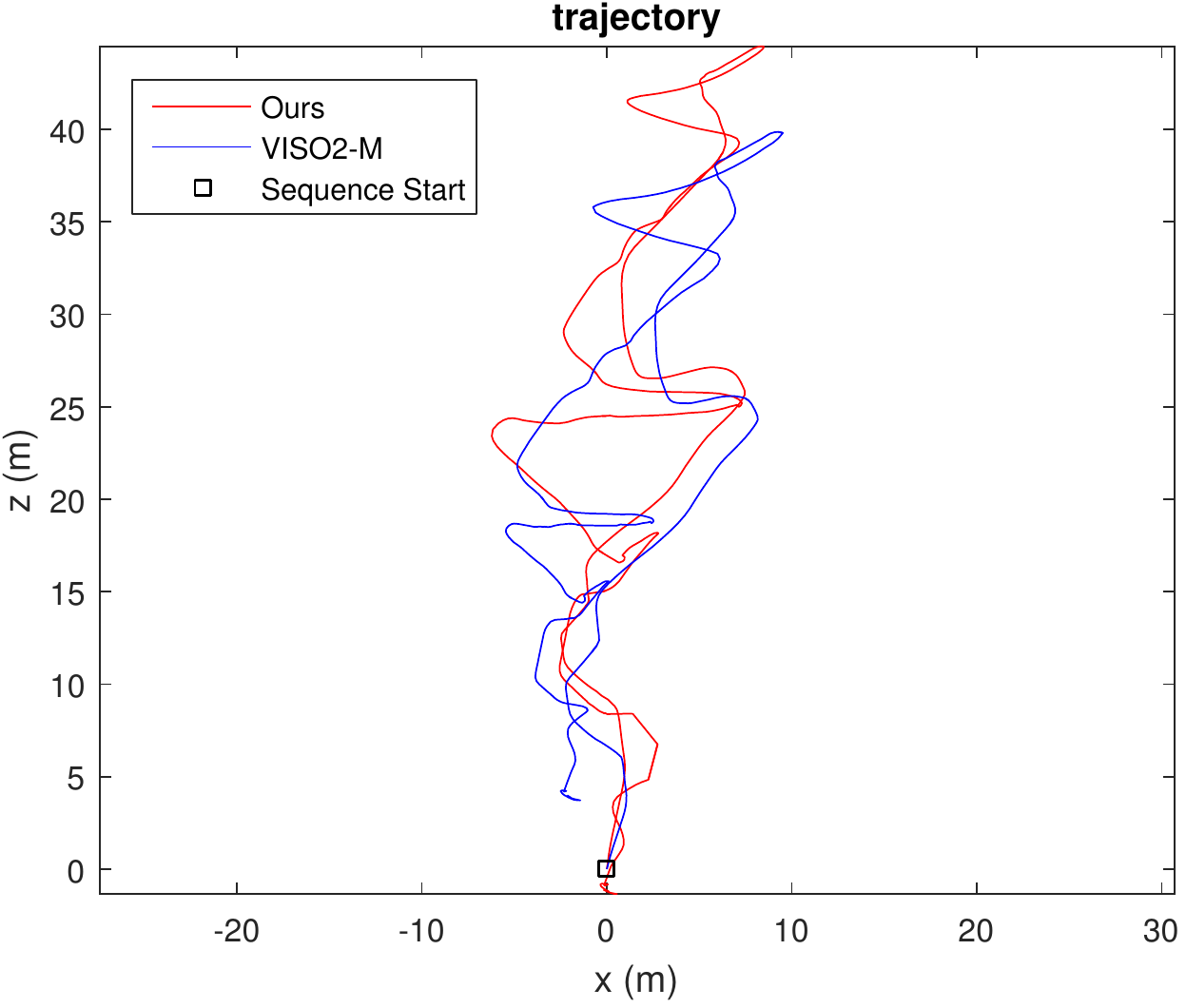}} 
	{\includegraphics[width=0.32\textwidth]{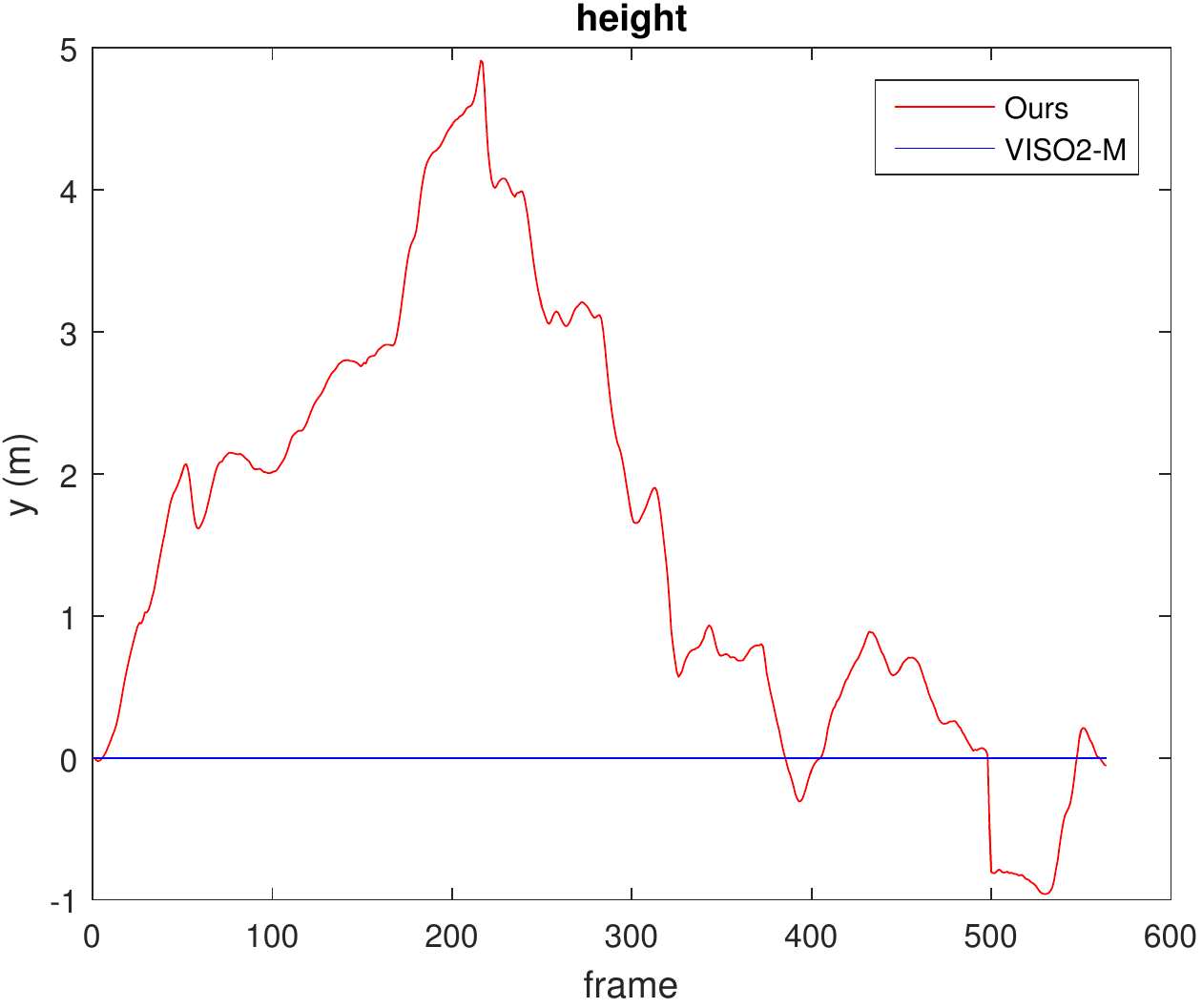}}
	{\includegraphics[width=0.325\textwidth]{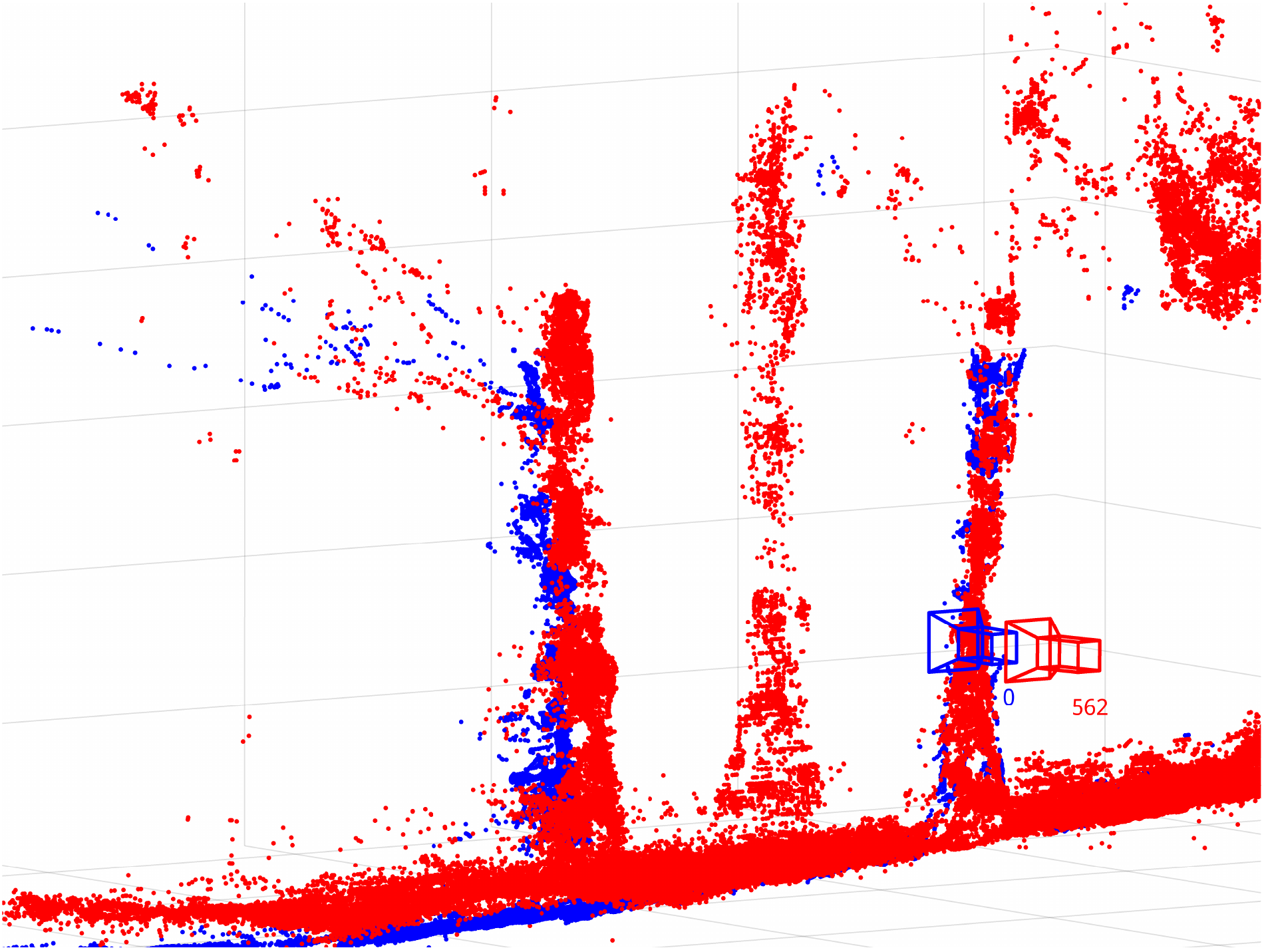}}
	\caption{\label{fig:aerial_fast} Fast moving UAV video result. From left to right: (a) the estimated trajectory, (b) estimated UAV height (zero at starting height, and positive is downwards), (c) our 3D reconstruction result of the first frame (blue) and last frame (red). }
\end{figure*}

\begin{table}
\centering
\begin{tabular}{|c|c|}
\hline
Method & Distance of farthest point to origin (m) \\
\hline
GPS & $45.81$ \\ % \pm 4.24
Ours & $45.50$ \\ % 44.62 (before CL)
VISO2-M & $40.93$ \\
\hline
\end{tabular}
\vspace{0.1cm}
\caption{\label{tab:farthest}Comparison of estimated distance of the farthest point from origin. }
\end{table}

\section{Conclusion}\label{sec:Conclusion}
This paper presented a new Mahalanobis eight-point algorithm using the dense optical flow. The full uncertainty of the optical flow was estimated in a principled manner by using the negative logarithm of a bivariate Gaussian distribution and fitting to the matching cost. The weighted eight-point algorithm optimized the Mahalanobis distance of each pixel correspondence to obtain a robust inter-frame motion estimate. With the SLAM pipeline of the front-end and back-end modules, the performance of the proposed method was evaluated using real datasets, demonstrating improved performance on a UAV platform compared to other state-of-the-art methods regarding accuracy and robustness. Future work includes the use of high-speed dense optical flow methods proposed as in \cite{Adarve16} for the real-time processing and collision avoidance in cluttered environment.

\balance

\end{document}